\documentclass[11pt]{article}

\usepackage[left=2.7cm,bottom=2.7cm,right=2.7cm,top=2.7cm]{geometry}

\usepackage{color,hyperref}

\usepackage[numbers]{natbib}
\newif\ifworkinprogress
\workinprogresstrue

\ifworkinprogress
\newcommand{\RM}[1]{\textcolor{blue}{\textbf{[RM] #1}}}
\newcommand{\HW}[1]{\textcolor{green}{\textbf{[HW] #1}}} 		  
\else
\newcommand{\RM}[1]{}
\newcommand{\HW}[1]{} 	
\fi

\usepackage{csquotes}
\usepackage{verbatim}

\usepackage{epsfig}
\usepackage{color}

\usepackage{balance}  
\usepackage{algorithmic}
\usepackage{algorithm}
\usepackage{multirow}
\usepackage{graphicx}
\usepackage{amsmath}
\usepackage{amssymb}
\usepackage{amsfonts}
\usepackage{xspace}
\usepackage{url}
\usepackage{bbm}
\usepackage{stmaryrd}

\usepackage{microtype}
\usepackage{subfig}
\usepackage{booktabs} 
\usepackage{adjustbox}
\usepackage{dblfloatfix}
\usepackage{tikz}

\usepackage{pifont}
\usepackage{mathtools}
\usepackage{amsthm}
\usepackage{makecell}
\usepackage[capitalize,noabbrev]{cleveref}

\usetikzlibrary{shapes.geometric, arrows}
\usetikzlibrary{positioning}
\usetikzlibrary{decorations.pathreplacing,calligraphy}

\newcommand{\R}{\mathbb{R}}

\newcommand{\lfl}{\lfloor}
\newcommand{\rfl}{\rfloor}

\newcommand{\lt}{\left}
\newcommand{\rt}{\right}

\newcommand{\td}{\tilde}
\newcommand{\wtd}{\widetilde}

\newcommand{\cI}{\mathcal{I}}
\newcommand{\cJ}{\mathcal{J}}

\newcommand{\cT}{\mathcal{T}}

\newcommand{\cX}{\mathcal{X}}
\newcommand{\cY}{\mathcal{Y}}

\newcommand{\AL}{\texttt{AL}}

\newcommand{\cmark}{\ding{51}}
\newcommand{\xmark}{\ding{55}}
\newcommand{\lsub}{L}
\newcommand{\rsub}{R}
\newcommand{\rootsub}{O}
\newcommand{\llsub}{LL}
\newcommand{\lrsub}{LR}
\newcommand{\rlsub}{RL}
\newcommand{\rrsub}{RR}
\newcommand{\alg}{\texttt{Quant-BnB}}

\theoremstyle{plain}
\newtheorem{theorem}{Theorem}[section]
\newtheorem{proposition}[theorem]{Proposition}
\newtheorem{lemma}[theorem]{Lemma}

\theoremstyle{definition}

\theoremstyle{remark}

\begin{document}
\title{Quant-BnB: A Scalable Branch-and-Bound Method for Optimal Decision Trees with Continuous Features
\author{
	Rahul Mazumder\thanks{MIT Sloan School of Management, Operations Research Center and MIT Center for Statistics ({email: rahulmaz@mit.edu}).}
	\and 
	Xiang Meng\thanks{MIT Operations Research Center (email: mengx@mit.edu)}
		\and
		Haoyue Wang\thanks{MIT Operations Research Center (email: haoyuew@mit.edu).}
}
}
\date{}
	\maketitle
	\begin{abstract}
Decision trees are one of the most useful and popular methods in the machine learning toolbox.  In this paper, we consider the problem of learning optimal decision trees, a combinatorial optimization problem that is challenging to solve at scale. A common approach in the literature is to use greedy heuristics, which may not be optimal. Recently there has been significant interest in learning optimal decision trees using various approaches (e.g., based on integer programming, dynamic programming)---to achieve computational scalability, most of these approaches focus on classification tasks with binary features. In this paper, we present a new discrete optimization method based on branch-and-bound (BnB) to obtain optimal decision trees. Different from existing customized approaches, we consider both regression and classification tasks with continuous features. The basic idea underlying our approach is to split the search space based on the quantiles of the feature distribution---leading to upper and lower bounds for the underlying optimization problem along the BnB iterations. Our proposed algorithm \texttt{Quant-BnB} shows significant speedups compared to existing approaches for shallow optimal trees on various real datasets.
\end{abstract}

\section{Introduction}\label{sect:intro}
A decision tree is a classic machine learning predictive tool with a flowchart-like structure that allows users to derive interpretable decisions.  
Combined with its effectiveness in solving classification and regression tasks, it is an immensely useful tool in 
domains where interpretability is of great importance. 
{Despite its appeal, the task of learning an optimal decision tree (with smallest training error)} is $\mathcal{NP}$-hard~\cite{laurent1976constructing} and thus computationally challenging. Therefore, greedy methods such as CART \cite{breiman1984classification} and ID3 \cite{quinlan1986induction} are popular choices. They construct decision trees through a top-down approach. Starting from the root node, data is iteratively split into subsets according to local objectives. In spite of high efficiency, greedy heuristics may not lead to an optimal solution, possibly resulting in suboptimal predictive performance~\cite{bertsimas2017optimal}.

In recent years, 
there have been major advances in exploring optimization methods to construct {\emph{optimal}}\footnote{In this paper,  {\emph{optimal}} refers to a global optimal solution to the optimization problem associated with learning a decision tree.} decision trees. \cite{verwer2019learning,gunluk2021optimal} explore mixed integer programming (MIP) approaches to learn 
optimal trees with a fixed depth. \cite{aglin2020learning,JMLR:v23:20-520} propose interesting dynamic programming (DP) approaches for optimal decision trees. 
Despite impressive methodological advances, optimal tree-learning approaches face the following challenges:~(\textbf{i})~MIP formulations appear to have limited scalability. \cite{verwer2019learning} report that a MIP solver cannot solve an optimal tree of depth $2$ with less than 1000 observations and 10 features in 10 minutes.~(\textbf{ii})~Most state-of-the-art algorithms \cite{JMLR:v23:20-520,aghaei2021strong,mctavish2022fast} consider datasets with binary features (i.e, every feature is $\{0,1\}$) rather than continuous ones.
Algorithms for optimal trees with continuous features are much less developed\footnote{While it is possible to convert continuous features to binary features, this may result in a large number of features, which leads to large computation times---see our experiments for details. To achieve scalability, it may be more beneficial to have tailored approaches for continuous features.}---our goal in this paper is to bridge this gap in the literature.

In this work we take a step towards addressing the aforementioned shortcomings by developing a novel branch-and-bound (BnB) algorithm for the computation of shallow optimal trees (e.g. depth=$2,3$). In contrast to earlier approaches, our algorithm can handle both classification and regression tasks and is designed to directly handle continuous features (including a mix of continuous and binary features). In a nutshell, our proposed algorithm \alg~utilizes the quantiles of the feature distribution to decompose the search space into sub-regions---these are subsequently used to generate lower bounds and upper bounds on the optimal value, and prune sub-optimal regions of the search-space. To our knowledge, \alg~ is the first standalone method (i.e, does not rely on proprietary optimization solvers) for optimal classification/regression trees that directly applies to datasets with continuous features. We show that \alg~computes the optimal solution (cf. Section \ref{sect:correctness}), and achieves significant empirical improvements compared to existing methods (cf Section \ref{sect:exp}). 
A Julia implementation of our code is open-sourced on GitHub\footnote{{\small{\url{https://github.com/mengxianglgal/Quant-BnB}}}}.

\begin{table}[h!]
\begin{adjustbox}{width=0.6\columnwidth,center}
\begin{tabular}{lcccc} \\ \Xhline{3\arrayrulewidth}
Paper & \textit{Opt} & \textit{Feat}   & \textit{Task}  & Model\\ \hline
Carreira-Perpin{\'a}n and Tavallali~\cite{carreira2018alternating} & \xmark &\cmark&  \textit{R}/\textit{C} & TAO\\
 Bertsimas and Dunn~\cite{bertsimas2019machine} &  \cmark &\cmark&  \textit{R}/\textit{C} & MIP\\
Verwer and Zhang~\cite{verwer2019learning} &  \cmark &\cmark&  \textit{C} & MIP\\
Aglin et al.~\cite{aglin2020learning} & \cmark& \xmark & \textit{C} & DP\&BnB\\
Demirovi{\'c} et al.~\cite{JMLR:v23:20-520} & \cmark& \xmark & \textit{C}& DP\\
Lin et al.~\cite{lin2020generalized} & \cmark &\xmark & \textit{C}& DP\&BnB\\
Aghaei et al.~\cite{aghaei2021strong} & \cmark &\xmark &\textit{R}/\textit{C}& MIP\\
Our approach & \cmark & \cmark &\textit{R}/\textit{C} & BnB\\  \Xhline{3\arrayrulewidth}
\end{tabular}
\end{adjustbox}
\caption{\small{Related works on decision tree optimization. \textit{Opt} indicates if the approach finds an optimal tree. \textit{Feat}  indicates whether the method works for continuous features (without binarizing features). \textit{Task} indicates what tasks the method can handle: \textit{R} for regression and \textit{C} for classification. TAO is an alternating optimization-based heuristic. The optimal methods are based on BnB, MIP (optimization solvers) and DP.
}}
\label{tab:relatework}
\end{table}
\textbf{Related Work:}
Various optimization techniques have been explored to learn high-quality decision trees~\cite{bennett1996optimal,dobkin1997induction,nijssen2007mining,farhangfar2008fast,nijssen2010optimal,carreira2018alternating}. A number of recent works explore MIP-approaches for optimal decision trees. For example, \cite{bertsimas2019machine} formulate learning optimal trees with a fixed depth as a MIP model. \cite{gunluk2021optimal} design an improved model with much fewer binary decision variables for classification problems with binary features. \cite{verwer2019learning} propose a model that works for numerical features with the same order of binary variables as in \cite{gunluk2021optimal}. \cite{NEURIPS2020_1373b284} present a MIP approach for optimal decision trees with hyperplane splits (aka oblique trees). 
Other MIP-based approaches have been proposed in~\cite{aghaei2019learning,aghaei2021strong}.
In addition to the MIP approach, SAT solvers have been explored to learn optimal decision trees~\cite{bessiere2009minimising,narodytska2018learning,hu2020learning}.

Another line of work explores pruning techniques to improve the efficiency of DP-based approaches.
\cite{angelino2017learning,chen2018optimization,hu2019optimal} solve the optimal sparse decision tree using analytical bounds on the optimal solution together with a customized bit-vector library. \cite{lin2020generalized} improve the efficiency of earlier approaches by using DP methods. \cite{aglin2020learning} utilize DP to compute better dual bounds during BnB search. \cite{JMLR:v23:20-520} show useful computational gains by using pre-computed information from sub-trees and hash functions. 
\cite{mctavish2022fast} design smart guessing strategies to improve the performance of BnB-based approaches.

As mentioned earlier, current approaches are unable to compute optimal classification/regression trees with continuous features at scale---a problem we address in this paper. Table~\ref{tab:relatework} presents a summary of the key characteristics of related existing approaches vis-a-vis our proposed method \alg.

\section{Preliminaries and Notations}
\subsection{Overview of optimal decision trees}
Consider a supervised learning problem with $n$ observations $\{(x_i,y_i)\}_{i\in [n]}$, each with $p$ features $x_i\in \cX \subseteq \R^p$ and response $y_i\in \cY$.
A decision tree recursively partitions the feature space $\cX$ into a number of hierarchical, disjoint regions, and makes a prediction for each region. In this paper, we focus on binary decision trees (i.e., every non-terminal node splits into left and right children) with axis-aligned splits. See \cite{breiman1984classification} for further details.

For a decision tree $T$ and feature vector $x$, let $T(x) \in  \cY$ denote the corresponding prediction. Given a loss function $\ell(\cdot, \cdot)$ on $\cY \times \cY$, 
and a family $\cT$ of decision trees, an \textit{optimal decision tree} is a global optimal solution to the following optimization problem:
\begin{equation}\label{opt-tree1}
\min\nolimits_{T\in \cT} ~~ \sum\nolimits_{i=1}^n \ell(y_i, T(x_i)).
\end{equation}
For regression problems with a scalar output, we can take $\cY \subset \R$ and 
$\ell$ as the squared loss: $\ell_{SE}(y, \hat y) := (y - \hat y)^2$. We also consider extensions to multivariate continuous outcomes with $y \in \cY \subset \R^m,\,m \geq 1$ and $\ell_{SE}(y, \hat y) := \| y - \hat y\|^2$.
For classification problems with $C$ classes, 
one can take $\cY = [C] := \{1,2,...,C\}$, and
$\ell$ to be the 0-1 (or mis-classification) loss i.e., $\ell_{01}(y, \hat y)=\mathbf{1}(y \neq \hat{y})$.

\subsection{Data types for the feature space}
In this paper, we consider the general case where $x_i$ contains continuous and possibly binary features. If a feature $f$ is continuous, then $\{x_{i,f}\}_{i=1}^n$ can contain at most $n$ different values. Our approach also applies to the setting where the number of distinct values of $f$ is much smaller than $n$. Recall that if $f$ is a binary feature then $x_{i,f} \in \{0,1\}, i \in [n]$.
As mentioned earlier, when the features are all binary, 
the optimal decision tree can be computed quite efficiently as shown in recent works (cf Section \ref{sect:intro}). 
To gather some intuition, note that, if all features are binary, computing an optimal regression tree of depth $d$ costs $O(n p^d)$ operations by exhaustive search---this can be acceptable even when $n$ is large (e.g. $n \approx 10^4 \sim 10^5$) but $d$ is small (e.g. $d =2,3$). 
In contrast, if all features are drawn from a continuous distribution, then an exhaustive search would cost $O(n^d p^d)$ (almost surely)---this may be computationally intractable for the same values of $n,d$.
In this paper, our focus is to propose an efficient BnB framework for the challenging case when the features are continuous---a 
topic that seems to be less studied when compared to the case where all features are binary.

\subsection{Notations}

For a feature $f\in [p]$, let 
$u(f) \in [n]$ be the number of distinct elements in $\{x_{i,f}\}_{i=1}^n$---i.e., the number of unique values assumed by feature $f$ in the training data. We let $w^f_1 < w^f_2 <\dots < w^f_{u(f)}$ denote these distinct values. In addition, we set $w^f_0=-\infty$ and $w^f_{u(f)+1} = +\infty$. 
For any integer $t$  with $0 \leq t \leq u(f)$, let $\td w^f_t := (w^f_t + w^f_{t+1})/2$. 
For Problem~\eqref{opt-tree1}, it suffices to consider candidate trees with all splitting thresholds located in the set $ \{ \td w^f_{t} \}_{f\in [p], 0\le t \le u(f) }$
(Note that  different splitting thresholds in the interval $(w^f_{t}, w^f_{t+1})$ give the same routing decision on the training set, so we choose the mid-point $\td w^f_{t}$ as a candidate threshold). As we consider axis-aligned splits, 
each splitting node of a tree
can be described by a tuple $ (f, t)$, 
where $f\in [p]$ is the splitting feature, and $ 0 \le t\le u(f)$ is an integer indicating that the splitting threshold is $\td w_t$. We say $(f,t)$ is the \textit{splitting rule} of this node. 
For a set $\cI \subseteq [n]$, a feature $f\in[p]$ and two integers $a,b$ with $0\le a\le b\le u(f)$, define
$$\cI^f_{[a,b]}~:=~\{i\in \cI~|~ \td w^f_a \le x_{i,f}\le \td w^f_b \}$$
to be the set of sample indices $i$ for which $x_{i,f}$ lies between $\td w^f_a$ and $\td w^f_b$. 

For $A, B >0$, we use the notations $A = O(B)$ or $A\lesssim B$ to denote that there exists a universal constant $C$ such that $A \le C B$; and use the notation $A = \wtd O(B)$ if $A = O(B\log(n))$, where $n$ is the number of samples.

\subsection{Preliminaries for decision tree with depth 2}
The algorithms we present in this paper are capable of solving shallow (e.g., $d=2$ or $3$) optimal trees within practical runtimes. 
We present an in-depth description of our proposed 
approach~\alg~for the case $d=2$. Section~\ref{sect:extension} presents a sketch of how it can be extended to deeper trees.

An optimal tree of depth 2 is a solution to the following problem
\begin{equation}\label{eq:opt}
\min_{T\in \cT_2} ~~ \sum_{i=1}^n \ell(y_i, T(x_i)),
\end{equation}
where, $\cT_2$ is the set of all decision trees with depth $2$ whose splitting thresholds are in $ \{ \td w^f_{t} \}_{f\in [p], 0\le t \le u(f) }$.

For a tree $T\in\cT_2$, let $(f_{\rootsub}(T),t_{\rootsub}(T))$, $(f_{\lsub}(T),t_{\lsub}(T))$ and $(f_{\rsub}(T),t_{\rsub}(T))$ denote the splitting rules at the root node ($\rootsub$), the left child ($\lsub$) and right child ($\rsub$) of the root node respectively. 
Given $f_0,f_1,f_2\in[p] $ and two integers $a,b$ with $0\le a\le b\le u(f_0)$, 
we define
 { \begin{equation*}\label{eq:defT}
 \small
     \begin{aligned}
 \cT_2(f_0,[a,b],f_1,f_2):=
 \Big\{ 
 T\in \cT_2  \Big| 
 f_{\rootsub}(T) = f_0, \, t_{\rootsub}(T)\in[a, b], f_{\lsub}(T) = f_1, \ f_{\rsub}(T)= f_2
 \Big\} 
     \end{aligned}
 \end{equation*}}
 to be the set of trees in $\cT_{2}$ whose root node, left child and right child splitting features are $f_0, f_1$ and $f_2$ respectively, and the splitting threshold of the root node is between $\td w_a$ and $\td w_b$. 
Given $F_1,F_2\subseteq [p]$, define 
\begin{equation}
\cT_2(f_0,[a,b],F_1,F_2) ~:=~ \bigcup_{f_1\in F_1,f_2\in F_2} \cT_2(f_0,[a,b],f_1,f_2). \nonumber
\end{equation} 
We adopt the convention that $\cT_2(f_0,[a,b],F_1,F_2) = \emptyset$ if $F_1$ or $F_2$ is empty. 
Note that here $a$ and $b$ are integers indicating that the splitting threshold of $f_0$ lies in $\{\td w_{a}^{f_0}, \td w_{a+1}^{f_0},, ...., \td w_{b}^{f_0}\}$.
See the appendix for an illustrative example. 

For a given subset of samples $\cI \subseteq  [n]$, define
\begin{equation}
    L_0(\cI) ~:=~ \min\nolimits_{y\in \cY}~ \sum\nolimits_{i\in \cI} \ell( y_i , y)
\end{equation}
to be the loss of the best constant fit to $\{y_{i}\}_{i\in\cI}$.
As concrete examples, 
for regression problem with $\cY = \R^m$ and $\ell = \ell_{SE}$, it holds
\begin{equation}\label{L0-reg}
L_0(\cI) \ = \
\sum\nolimits_{i\in \cI} \big\| 
y_i - ({1}/{|\cI|}) \sum\nolimits_{j\in \cI} y_j
\big\|^2. 
\end{equation}
For classification problem with $\cY = [C]$ and loss $\ell = \ell_{01}$, it holds
\begin{equation}\label{L0-cls}
L_0(\cI) \ = \
\min\nolimits_{c\in [C]} \sum\nolimits_{i\in \cI} \textbf{1}_{\{y_i \neq c\}}.
\end{equation}

Note that $L_0(\cI)$ can be viewed as
the minimum loss of a ``depth-0" decision tree with observations in $\cI$. 
Similarly, 
we define a function $L_1$ to be 
the best possible objective value for a depth-1 decision tree (a ``stump") when using observations in $\cI$ and a feature $f\in [p]$:
\begin{equation}\label{eq:defL1}
L_1 (\cI, f) ~:=~ \min_{0 \le t \le u(f)}  \big\{  L_0(\cI^f_{[0,t]}) +  L_0(\cI^f_{[t,u(f)]}) \big\}.
\end{equation}
Note that for a given feature $f$ and a set of indices $\cI$, if $L_0$ is given by \eqref{L0-reg} or \eqref{L0-cls}, then $L_1 (\cI, f) $ can be computed within $O(|\cI| \log |\cI|) \le \wtd O(|\cI|)$ operations (the log-factor is from a sorting operation). 

Let us consider $f_0, f_1, f_2 \in [p]$, integers $a,b$ with $0 \le a \le b \le u(f_0)$ and $\cI \subseteq [n]$. 
Extending the definitions of $L_0$ and $L_1$ above, we now consider the best objective for a depth-2 decision tree. Define 
\begin{equation}\label{def:L2}
 L_2(\cI, f_0, [a,b], f_1, f_2)  
:= \min_{a \le t \le b} \big\{  L_1(\cI^{f_0}_{[0,t]},f_1) +  L_1(\cI^{f_0}_{[t,u({f_0})]},f_2) \big\}
\end{equation}
to be the minimum value of $\sum_{i\in \cI} \ell(y_i,T(x_i))$ over all depth-2 trees $T \in \cT_2(f_0, [a,b], f_1, f_2)$. 

To simplify the notation above, define the parameter space
\begin{equation}\label{def:Phi}
\begin{aligned}
\Phi := &\Big\{ (f_0, [a,b], f_1, f_2)~|~ f_0, f_1,f_2 \in [p], \\
&~~~~ a,b \text{ are integers with } 0\le a \le b \le u(f_0) \Big\}. 
\end{aligned}
\end{equation}
Then for any $\phi = (f_0, [a,b], f_1, f_2) \in \Phi$, and $\cI\subseteq [n]$, we use the 
short-hand notations 
\begin{eqnarray}
\cT_2(\phi) &=& \cT_2(f_0, [a,b], f_1, f_2), \nonumber\\
L_2(\cI, \phi) &=& L_2(\cI, f_0, [a,b], f_1, f_2). \nonumber 
\end{eqnarray}

\section{Upper and Lower Bounds for $L_2(\cI,\phi)$}

For any $\phi \in \Phi$ and $\cI \subseteq [n]$, recall that $L_2(\cI,\phi)$ is the best objective value of \eqref{eq:opt} across trees in $\cT_2(\phi)$ with samples in $\cI$. In this section, we discuss upper and lower bounds for $L_2(\cI,\phi)$---the costs for computing these bounds are lower than directly computing $L_2(\cI,\phi)$. \texttt{Quant-BnB} critically makes use of these upper and lower bounds while performing BnB (cf Section \ref{sect:alg}).

\subsection{Upper bounds for $L_2(\cI,\phi)$}

Given $\phi = (f_0, [a,b], f_1,f_2) \in \Phi$ and $\cI \subseteq [n]$, we compute an upper bound of $L_2(\cI, \phi)$ by making use of the quantiles of the features. 
For an integer $s$ with $1\le s \le b-a$, we say that a set of integers $(t_0, t_1, \dots, t_s)$ are \textit{almost s-equi-spaced} in the interval $[a,b]$ if $t_0 = a$, $t_s = b$ and $t_j = \lfloor a + ({j}/{s}) (b-a)\rfloor$ for $1 \le j\le s-1$. Given such a sequence $(t_0, t_1, \dots, t_s)$, we define
\begin{equation}
    V_s (\cI, \phi) := \min_{0\le j \le s} \Big\{ L_1(\cI_{[0,t_j]}^{f_0}, f_1 )  +  L_1 ( \cI_{[t_j,u(f_0)]}^{f_0}, f_2 ) \Big\}.
    \nonumber
\end{equation}
It follows that $V_s (\cI, \phi) \ge L_2(\cI, \phi)$ for all $\phi = (f_0, [a,b], f_1,f_2) \in \Phi$ satisfying $b-a \ge s$; and hence $V_{s}$ is an upper bound to $L_2$. We note that quantile-based methods are commonly used as a heuristic in decision tree algorithms (e.g., XGBoost \cite{chen2016xgboost}). Our work differs---as discussed below, we make use of this quantile-based approach to obtain an optimal decision tree.

\subsection{Lower bounds for $L_2(\cI, \phi)$}\label{sect:lower-bounds}
We present some lower bounds for $L_2(\cI, \phi)$. The lower bounds along with the upper bounds discussed earlier, form the backbone of our BnB procedure. 

\noindent \textbf{Lower bound~1:}
The first lower bound we consider is obtained by sorting the values of a feature, and dropping the values lying in an interior sub-interval. 
Given any $\phi = (f_0, [a,b], f_1,f_2) \in \Phi$ and $\cI \subseteq [n]$, 
define 
\begin{equation}\label{def:W0}
    W_0(\cI, \phi) = W_0(\cI, f_0, [a,b], f_1, f_2) 
 := L_1( \cI_{[0,a]}^{f_0} , f_1 ) +
    L_1( \cI_{[b,u(f_0)]}^{f_0} , f_2 ).
\end{equation}
We can interpret $W_0(\cI, \phi)$ as follows: the samples in $\cI_{[0,a]}^{f_0} $ are routed to the left subtree yielding a loss $L_1( \cI_{[0,a]}^{f_0} , f_1 )$; the samples in $\cI_{[b,u(f_0)]}^{f_0}$ are routed to the right subtree with a loss $L_1( \cI_{[b,u(f_0)]}^{f_0} , f_2 )$; and the samples in $ \cI_{[a,b]}^{f_0} $ are ``dropped" (i.e., do not enter any leaf node). Lemma \ref{lemma:key-ineq} shows that $W_0$ is a lower bound for $L_{2}$ i.e., $W_0(\cI, \phi) \le L_2(\cI, \phi)$ for all $\phi \in \Phi$.
Furthermore, $W_0(\cI,\phi)$ is easier to compute 
than $L_2(\cI, \phi)$. 
Figure~\ref{fig:engage} presents a schematic diagram illustrating computation of the lower bound $W_0$. 

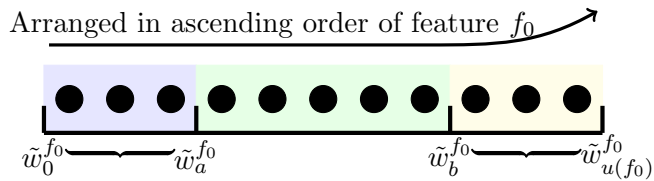
\begin{figure}[h]
\centering
\begin{tikzpicture}[scale=0.9]
   \draw [fill=blue!10,blue!10] (-0.375,-0.5) rectangle (1.875,0.5);
   \draw [fill=green!10,green!10] (1.875,-0.5) rectangle (5.625,0.5);
   \draw [fill=yellow!10,yellow!10] (5.625,-0.5) rectangle (7.875,0.5);
   
   \draw [fill] (0,0) circle [radius=0.2];
   \draw [fill] (0.75,0) circle [radius=0.2];
   \draw [fill] (1.5,0) circle [radius=0.2];
   \draw [fill] (2.25,0) circle [radius=0.2];
   \draw [fill] (3,0) circle [radius=0.2];
   \draw [fill] (3.75,0) circle [radius=0.2];
   \draw [fill] (4.5,0) circle [radius=0.2];
   \draw [fill] (5.25,0) circle [radius=0.2];
   \draw [fill] (6,0) circle [radius=0.2];
   \draw [fill] (6.75,0) circle [radius=0.2];
   \draw [fill] (7.5,0) circle [radius=0.2];
   \draw [ultra thick] (-0.375,-0.5) -- (7.875,-0.5);
   \draw [ultra thick] (-0.375,-0.5) -- (-0.375,-0.1);
   \draw [ultra thick] (7.875,-0.5) -- (7.875,-0.1);
   \draw [ultra thick] (1.875,-0.5) -- (1.875,-0.1);
   \draw [ultra thick] (5.625,-0.5) -- (5.625,-0.1);
   
   \draw[->,very thick] (-0.3,0.8) to (5.3,0.8) to [out=0,in=210] (7.8,1.3);
   \node at (3,1.1) {Arranged in ascending order of feature $f_0$};
   \node at (-0.375,-0.85) {$\td w^{f_0}_0$};
   \node at (1.875,-0.85) {$\td w^{f_0}_{a}$};
   \node at (5.625,-0.85) {$\td w^{f_0}_{b}$};
   \node at (8.1,-0.85) {$\td w^{f_0}_{u(f_0)}$};
   \draw [decorate,ultra thick,
    decoration = {calligraphic brace}] (1.55,-0.75) --  (-0.05,-0.75);
    \draw [decorate,ultra thick,
    decoration = {calligraphic brace}] (7.55,-0.75) --  (5.95,-0.75);
   \end{tikzpicture} 
\caption{Figure showing sorting of the unique values in feature $f_0$ into intervals $[\td w^{f_0}_0, \td w^{f_0}_{a}]$, $[\td w^{f_0}_{a},\td w^{f_0}_{b}]$ and $[\td w^{f_0}_{b},\td w^{f_0}_{u(f_0)}]$. In the definition of $W_0$, samples with $x_{i,f_0}$ between $\td w^{f_0}_a$ and $\td w^{f_0}_b$ (green part in the figure) are dropped. 
}
\label{fig:engage}
\end{figure}

\noindent \textbf{Lower bound~2:} The second lower bound we consider is obtained by using a subset of samples in the middle to fit another depth-2 tree. Given any $\phi = (f_0, [a,b], f_1,f_2) \in \Phi$ and $\cI\subseteq [n]$, define
\begin{equation}\label{def:W2}
 W_2(\cI, \phi) = W_2(\cI, f_0, [a,b], f_1, f_2) 
:= 
    W_0(\cI, \phi) + L_2( \cI^{f_0}_{[a,b]}, \phi ).
\end{equation}
Note that in the definition of $W_2$, the samples in $\cI_{[a,b]}^{f_0}$ are used to fit another depth-2 tree in $\cT_2(\phi) $, and yield an additional loss term $L_2( \cI^{f_0}_{[a,b]}, \phi )$. It can be proved that $W_2(\cI,\phi) \le L_2(\cI, \phi)$ for all $\phi \in \Phi$ (see Lemma \ref{lemma:key-ineq}). 
Note that $W_2(\cI,\phi)$ is a better lower bound than $W_0(\cI,\phi)$, but has a higher computational cost---see Section \ref{sect:correctness} for details. 

\noindent \textbf{Lower bound~3:} The third lower bound we introduce below combines the above ideas underlying the computation of $W_0$ and $W_2$. 
To introduce this lower bound, we need some additional notations. Given a $\phi = (f_0, [a,b], f_1,f_2) \in \Phi$; and almost $s'$-equi-spaced integers $(t_0, t_1, \dots, t_s')$ in $[a,b]$ (for an integer $s' \le b-a$), define
\begin{equation}\label{eq:defhatL2}
    \widehat L_2(\cI, \phi, s') = 
    \widehat L_2 (\cI, f_0, [a,b], f_1, f_2, s') 
    := \min_{1 \le j \le s'} \big\{  L_1(\cI^{f_0}_{[0,t_{j-1}]},f_1) 
     +  L_1(\cI^{f_0}_{[t_{j},u({f_0})]},f_2) \big\}. 
\end{equation}
 Note that in the $j$-th term of the minimum in~\eqref{eq:defhatL2}, we drop the 
observations in $\cI^{f_0}_{[t_{j-1}, t_j]}$. 
The following lemma shows that $\widehat L_2 (\cdot, \cdot, s')$ is a lower bound of $L_2(\cdot, \cdot)$. 

\begin{lemma}\label{lemma:L2-lb}
For any $\cI\subseteq [n]$, any $\phi = (f_0, [a,b], f_1, f_2) \in \Phi$, and any $s\le b-a$, it holds 
 \begin{equation}
 \widehat L_2 (\cI, \phi, s') 
 \le 
 L_2(\cI, \phi).
 \end{equation}
\end{lemma}
The proof of Lemma \ref{lemma:L2-lb} is presented in the appendix. Note that the computation of $\widehat L_2 (\cI, \phi, s') $ is easier than the computation of $L_2(\cI, \phi)$.

Using $\widehat L_2$, we can introduce the third lower bound for $L_2(\cI, \phi)$. 
Given any $\phi = (f_0, [a,b], f_1,f_2) \in \Phi$ and any integer $s' \le b-a$, 
define 
\begin{equation}\label{def:W1}
 W_{1,s'}(\cI,\phi) = 
W_{1,s'}(\cI,f_0, [a,b], f_1, f_2) 
:= 
    W_0(\cI,\phi) + \widehat L_2( \cI^{f_0}_{[a,b]}, \phi, s' ).
\end{equation}
Note that in the definition above, the samples in $\cI^{f_0}_{[a,b]}$ are used to compute a term $\widehat L_2( \cI^{f_0}_{[a,b]}, \phi, s' )$. As a result, the value $W_{1,s'}(\cI,\phi)$ is larger than $W_0(\cI,\phi)$ and smaller than $ W_2(\cI,\phi)$ (due to Lemma \ref{lemma:L2-lb}). 

The following lemma justifies that $W_0$, $W_{1,s'}$ and $W_2$ are indeed lower bounds of $L_2$. 
\begin{lemma}\label{lemma:key-ineq}
For any $\cI \subseteq [n]$, $\phi \in \Phi$ and $s'\le b-a$, it holds
\begin{equation}
W_0(\cI,\phi) \le W_{1,s'}(\cI,\phi) \le W_2(\cI,\phi) \le 
L_2(\cI, \phi). \nonumber
\end{equation}
\end{lemma}
The proof of Lemma \ref{lemma:key-ineq} is presented in the appendix. 
Lemma \ref{lemma:key-ineq} shows that 
$W_0$, $W_{1,s'}$ and $W_2$ are lower bounds for $L_2$; $W_0$ is weakest and $W_2$ is the tightest of the three bounds.
See Section \ref{sect:correctness} for further discussions on the computational costs of these three lower bounds.

\section{A Branch and Bound Method for Optimal Trees with Depth $2$}\label{sect:alg}
In this section, we describe our algorithm \alg~to solve problem \eqref{eq:opt} to optimality. \alg~maintains and updates a search space  represented as disjoint unions of sets of the form $\cT_2(f_0,[a,b],F_1,F_2)$. In Section \ref{sect:quantile-decomp}, we first present a proposition illustrating how we perform quantile-based pruning on a set of the form $\cT_2(f_0,[a,b],F_1,F_2)$. In Section \ref{sect:overall strategy} we present the overall framework of \texttt{Quant-BnB}. In Section \ref{sect:correctness} we discuss the computational cost of the algorithm.

\subsection{A quantile-based pruning procedure }\label{sect:quantile-decomp}
Suppose we are given $f_0\in [p]$, integers $a,b$ with $0 \le a \le b \le u(f_0)$ and $F_1,F_2 \subseteq [p]$. 
We focus on the subset of trees $\cT_2(f_0,[a,b],F_1,F_2)$, and discuss a quantile-based method to replace this collection by a smaller subset containing an optimal solution to~\eqref{eq:opt}.

Let $(t_0,...,t_s)$ be almost $s$-equi-spaced in $[a,b]$. 
The following proposition states a basic strategy for pruning.
\begin{proposition}\label{prop: quantile-prune}
Let $W$ be a function on $2^{[n]}\times \Phi$ with $W(\cI, \phi) \le L_2(\cI,\phi)$ for all $ \phi \in \Phi$ and $\cI \subseteq [n]$; let $U$ be an upper bound of the optimal value of problem \eqref{eq:opt}. For each $j\in [s]$, define 
\begin{align}
    F_{1,j} &:= \big\{f_1\in F_1 ~\big|~ \min_{f_2\in F_2} W([n],\phi^j_{f_1,f_2}) \le U  \big\}, \label{eq:def-F1j} \\
F_{2,j} &:= \big\{f_2\in F_2 ~\big|~ \min_{f_1\in F_1} W([n],\phi^j_{f_1,f_2}) \le U  \big\},
\label{eq:def-F2j}
\end{align}
where $\phi^{j}_{f_1,f_2} := (f_0, [t_{j-1}, t_j], f_1, f_2)$. 
Then any optimal solution of \eqref{eq:opt} is not in 
\begin{equation}\label{eq:prune1}
\cT_2(f_0, [a,b], F_1, F_2) \setminus \mathop{\cup}\limits_{j=1}^s \cT_2(f_0, [t_{j-1}, t_j], F_{1,j}, F_{2,j}).
\end{equation} 
\end{proposition}
The proof of Proposition \ref{prop: quantile-prune} is presented in the appendix. Possible choices of the lower bound $W(\cI,\phi)$ in Proposition \ref{prop: quantile-prune} are $W_0(\cI,\phi)$, $W_{1,s'}(\cI,\phi)$ and $W_2(\cI,\phi)$, as designed in Section \ref{sect:lower-bounds}.
If the assumptions of Proposition \ref{prop: quantile-prune} hold, we can replace the search space $\cT_2(f_0, [a,b], F_1, F_2)$ by a  smaller space
$\cup_{j=1}^s \cT_2(f_0, [t_{j-1}, t_j], F_{1,j}, F_{2,j})$, or equivalently, prune all the feasible solutions in \eqref{eq:prune1}. Note that it is possible that for some $j\in [s]$, the set $F_{1,j}$ or $F_{2,j}$ is empty, and hence the set $\cT_2(f_0, [t_{j-1}, t_j], F_{1,j}, F_{2,j})$ is also empty. In that case, we prune the trees in $\cT_2(f_0, [a,b], F_1, F_2)$ with splitting thresholds lying in $[t_{j-1}, t_j]$. 

\subsection{\texttt{Quant-BnB} framework}\label{sect:overall strategy}
We discuss the overall methodology of \texttt{Quant-BnB} to solve \eqref{eq:opt}. The proposed algorithm maintains and updates 
a set $\AL^{(k)}$ 
(short for ``alive") 
over iterations $k \geq 0$. $\AL^{(k)}$  contains tuples of the form
\begin{equation*}
(f_0, [a,b], F_1, F_2),
\end{equation*}
where $f_0\in [p]$; $a$ and $b$ are integers with $0\le a\le b\le u(f_0)$; and $F_1, F_2 \subseteq [p]$. 
Initially (i.e., at $k=0$), all the trees in $\cT_2$ are ``alive", so we set 
\begin{equation}\label{def:AL0}
\AL^{(0)} = \mathop{\cup}\nolimits_{f_0=1}^p \lt\{   (f_0, [0,u(f_0)], [p], [p])   \rt\}.
\end{equation}
Intuitively speaking, as the iterations progress, we reduce the size of $\AL^{(k)}$, by removing tuples $(f_0, [a,b], F_1, F_2)$ that do not contain optimal solutions to~\eqref{eq:opt}. 
The algorithm also maintains and updates the best objective value that it has found so far, denoted by $U$. Initially, we set $U $ to be the value at any feasible solution of \eqref{eq:opt}.

At every iteration $k\ge 1$, to update $\AL^{(k)}$ from $\AL^{(k-1)}$, we first set $\AL^{(k)} = \emptyset $. 
The algorithm then checks all elements in $\AL^{(k-1)}$.
For an element $(f_0, [a,b], F_1, F_2)$ in $\AL^{(k-1)}$, if $b-a$ is less than a pre-specified integer $s$ (i.e., the number of trees in the space is sufficiently small), our algorithm performs an exhaustive search to examine all candidate trees in the space $\cT_2(f_0, [a,b], F_1, F_2)$. Otherwise, the algorithm conducts the following $2$ steps. Let $(t_0,...,t_s)$ be an almost $s$-equi-spaced sequence in $[a,b]$. 
\begin{itemize}
    \item (Step1: Update upper bound) Compute
    \begin{equation}
    U'= 
    \min_{f_1\in F_1, f_2\in F_2} \big\{
        V_s([n], f_0, [a,b], f_1, f_2)
        \big\}. \nonumber
    \end{equation}
    If $U' < U$, update $U \leftarrow U'$, and update the corresponding best tree. 
    
    \item (Step2: Compute lower bound and prune) For a function $W$ on $2^{[n]}\times \Phi$ satisfying $W(\cI,\phi) \le L_2(\cI,\phi)$ for all $ \phi \in \Phi$ and $\cI \subseteq [n]$, compute
     sets $ F_{1,j}$ and $F_{2,j}$ as in \eqref{eq:def-F1j} and \eqref{eq:def-F2j} (for all $j\in [s]$), and update 
    \begin{equation}
        \AL^{(k)} =  \AL^{(k)} \bigcup \Big( \cup_{j=1}^s \{(f_0, [t_{j-1}, t_j], F_{1,j}, F_{2,j})\} \Big). \nonumber
    \end{equation}
    
\end{itemize}
{Note that in Step 2 above, we need to compute values of $W([n], \phi_{f_1,f_2}^{j})$ as in \eqref{eq:def-F1j} and \eqref{eq:def-F2j}--these are lower bounds of $L_2([n], \phi_{f_1,f_2}^{j})$. Function $W$ can be taken as $W_0$, $W_{1,s'}$ or $W_2$, as introduced in Section \ref{sect:lower-bounds}. At the end of Step 2, the set $\cup_{j=1}^s \{(f_0, [t_{j-1}, t_j], F_{1,j}, F_{2,j})\}$ is added into $\AL^{(k)}$; this set replaces the tuple $ (f_0, [a,b], F_1, F_2)$ in $\AL^{(k-1)}$. In other words, all the trees that lie in $\cT_2(f_0, [a,b], F_1, F_2)$ but not in $\cup_{j=1}^s \cT_2(f_0, [t_{j-1}, t_j], F_{1,j}, F_{2,j}) $ are pruned.}

The main steps of the algorithm above are summarized in Algorithm \ref{alg:bb-depth2}. For illustration, a single iteration of \texttt{Quant-BnB} on a toy example is provided in appendix.

\begin{algorithm}[htb]
   \caption{\texttt{Quant-BnB} for depth-2 decision trees}
   \label{alg:bb-depth2}
\begin{algorithmic}
   \STATE {\bfseries Input:} data $\{(x_i,y_i)\}_{i=1}^n$, an integer $s\ge2$, an initial upper bound $U$~ and a feasible solution $\hat T$.
   \STATE Initialize $\AL^{(0)}$ as in \eqref{def:AL0}, and
   set $k=1$.
   \REPEAT
\STATE Set $\texttt{AL}^{(k)} = \emptyset$.
   \FOR{{\bfseries each} $( f_0, [a,b], F_1, F_2 )$ in $\AL^{(k-1)}$}
\IF {$b-a \le s$}
\STATE Use exhaustive search to get the bound 
\\$U' = \min_{f_1\in F_1, f_2\in F_2} ~L_2([n],f_0,[a,b],f_1,f_2)$.
\STATE Update $U=\min\{U,U'\}$ and accordingly, the current best solution $\hat T$.
\STATE {\bfseries continue}
\ENDIF
 \STATE Let $(t_0, t_1, \dots, t_{s})$ be almost $s$-equi-spaced in $[a,b]$. 
 \STATE Perform Steps 1 and~2. 
   \ENDFOR
   \STATE Update $k\leftarrow k+1$.
   \UNTIL{ $\texttt{AL}^{(k)}$  is empty}
 \STATE {\bfseries Output:} The optimal decision tree $\hat T$ and corresponding objective value $U$.
\end{algorithmic}
\end{algorithm}

\subsection{Correctness of~\alg, computational cost}\label{sect:correctness}
Theorem~\ref{thm:convergence} (see Appendix for proof) establishes that \alg~converges to the global optimum of~\eqref{eq:opt}.
\begin{theorem}\label{thm:convergence}
Algorithm \ref{alg:bb-depth2} terminates in at most $\lceil log_{s}(n)\rceil$ iterations and yields an optimal solution of \eqref{eq:opt}. 
\end{theorem}

\noindent{\bf Computational cost:} We discuss the computational cost of Steps 1 \& 2 for a given $(f_0, [a,b], F_1, F_2)$. 
To simplify the discussion, we consider the case when the number of unique values of feature $f_0$ i.e., $u(f_0)$ is $n$. 

Note that in Algorithm \ref{alg:bb-depth2}, a function $W$ satisfying $W(\cI, \phi) \le L_2(\cI,\phi)$ for all $ \phi \in \Phi$ and $\cI \subseteq [n]$ is needed. Candidates of the lower bound $W$ have been discussed in Section \ref{sect:lower-bounds}.
Different choices of $W$ have different computational costs. 
The simplest choice is to set $W = L_2$ directly, 
in which case Step 2 of the algorithm reduces to an exhaustive search over $\cT_2(f_0, [a,b], F_1, F_2)$, which is expensive. 
By Lemma \ref{lemma:key-ineq}, it is also possible to take $W$ to be $W_0$, $W_{1,s'}$ (for some proper integer $s'$) or $W_2$. 
We compare the computational cost of Step 1 -- Step 2 under these different choices, as shown below. 
\begin{lemma}\label{lemma: compare-cost}
Suppose $L_0$ is given by \eqref{L0-reg} or \eqref{L0-cls}. 
For a given $(f_0, [a,b], F_1, F_2)$, denote $ \td p := |F_1|+|F_2|$. Suppose $u(f_0)=n $. Let $s, s'$ be positive integers with $s'\cdot s \le b-a$. The computational cost of Steps~1-2 for different choices of the lower-bound function $W$ are as follows: \newline
(1) If $W = W_{0}$, the cost is bounded by $\wtd O(n\td p s)$. \newline
(2) If $W = W_{1,s'}$, the cost is bounded by $\wtd O(n\td p s + \td p s' {(b-a)})$.\newline
(3) If $W = W_{2}$, the cost is bounded by $\wtd O(n\td p s + \frac{(b-a)^2 \td p}{s})$. \newline
(4) If $W = L_2$, the cost is bounded by $\wtd O(n\td p (b-a))$.
\end{lemma}

Note that the assumption $s\cdot s' \le b-a$ is necessary to make sure the equi-spaced sequences are well-defined. 
By this assumption, we have $ n\td p s \le 
n\td p s + \td p s' {(b-a)} \le n\td p s + {(b-a)^2 \td p}/{s} \lesssim n\td p (b-a)$.
Lemma \ref{lemma: compare-cost} implies that
\begin{equation}\label{eq:lowerbound-qual}
    W_0 \prec W_{1,s'} \prec W_2 \prec L_2,
\end{equation}
where the notation ``$\bar{W} \prec \tilde{W}$" means that the cost (of Steps~1 \& 2) in using $W = \bar{W}$ is bounded by a constant multiple of the cost of using $W = \tilde{W}$. 
On the other hand, by Lemma  \ref{lemma:key-ineq}, it is known that 
\begin{equation}\label{eq:lowerbound-time}
W_0(\cI, \phi) \le W_{1,s'}(\cI, \phi) \le W_2(\cI, \phi) \le 
L_2(\cI, \phi) 
\end{equation}
for all $\phi \in \Phi $. Therefore, among these four choices of $W$, there is a tradeoff: the choice with lower computational time for Step 1 -- Step 2 will produce a less tight lower bound and may result in a case where fewer trees are pruned in this iteration. 
Empirically, we find that among these four choices, 
choosing $W = W_{1,s'}$ (for a proper $s'$) 
has the best performance in most cases (see Section \ref{sect:numexp1} for a comparison of these choices). A choice of $s'$ that appears to work well in practice is $s' \approx \frac{ns}{b-a} $, in which case the cost in Lemma \ref{lemma: compare-cost} (2) reduces to $\wtd O(n\td p s)$. With this choice of $s'$, the cost of 
Step 1--2 using $W = W_{1,s'}$ is the same (up to a constant multiple) as the cost of choosing $W = W_0$, but the former always provides a better lower bound. 

We note that when choosing $W = W_{2}$ or $W = L_2$, the cost of Algorithm \ref{alg:bb-depth2} is not linear in $n$. Indeed, initially ($k=0$), for any 
$(f_0, [a,b], F_1, F_2)$ in $\AL^{(0)}$, it holds $b-a = n$, and $\td p =2p$. So by Lemma \ref{lemma: compare-cost} (4) the cost of Steps~1--2 with $W = L_2$ is $\wtd O(n^2\td p )$; and with $W = W_2$ the cost is $\wtd O(nps + ({n^2p}/{s})) \ge \wtd O(n^{3/2}p)$.

\section{Extension to Deeper Trees}\label{sect:extension}

Algorithm \ref{alg:bb-depth2} can be generalized for the computation of optimal trees with a fixed depth $d$ ($d\ge 3$). 
We briefly discuss the case when $d=3$.
To compute an optimal tree of depth $3$, 
\alg~maintains and updates a set $\texttt{AL3}$ that contains tuples of the form 
\begin{equation}\label{example-tuple-d=3}
(f_0,[a,b],\Phi_1,\Phi_2),
\end{equation} 
where $f_0\in [p]$, $0\le a\le b\le u(f_0)$, and $\Phi_1, \Phi_2\subseteq \Phi$.
Recall that $\Phi$ is defined in \eqref{def:Phi}, which contains tuples corresponding to subsets of depth-2 trees. 
Each tuple \eqref{example-tuple-d=3} corresponds to a search space in which the elements meet the following conditions: the root node $(f_0,t_0)$ satisfies $t_0\in[a,b]$; also, the left and right branch of the root node, which are decision trees of depth $2$, 
are in $\cT_2(\phi_1)$ and $\cT_2(\phi_2)$ for some $\phi_1 \in \Phi_1$ and $\phi_2 \in\Phi_2$ respectively. 

To shrink the search space corresponding to \eqref{example-tuple-d=3}, we set up an almost equi-spaced sequence of integers and work with lower bounds for each smaller search space.
Due space limits, we present the details of the algorithm and related discussions in Section \ref{app-sect: depth-3}.

For the case $d\ge 4$, similar recursion can be applied to design BnB algorithms, but the computational cost increases especially when $p$ (\# of features) is large. Therefore, we recommend using our procedure for fitting optimal decision trees with $d\le 3$. 

So far, we only consider perfect binary trees (i.e., depth-$d$ trees that has exactly $2^d-1$ branch nodes). In practice, it is preferable to optimize over non-perfect trees to enhance generalization capability, especially for deeper trees. Note that we can modify \alg~to handle non-perfect trees by considering two cases for each node--it can be a branch node or a leaf--when calculating upper and lower bounds. This modification will not result in additional computation costs.

\section{Numerical Experiments}\label{sect:exp}

In this section, we study the performance of \texttt{Quant-BnB} in terms of runtime and prediction accuracy. In particular we study:~({i})~ differences in the efficiency of \texttt{Quant-BnB} using various methods to calculate the lower bound of $L_2(\cI,\phi)$  ~({ii})~computational cost of optimal trees (depths $2$ and $3$) on classification tasks compared to state-of-the-art algorithms ~({iii})~out-of-sample accuracy compared to heuristic methods. We present details of experimental setup and results on regression tasks in appendix Section~\ref{app-sect:exp}.

\textbf{Datasets and Computing Environment:} We collect 16 classification (binary and multi-class) and 11 regression datasets from UCI Machine Learning Repository \cite{Dua:2019}. All experiments are carried out on MIT's Engaging cluster on Intel Xeon 2.30GHz machine, with a single CPU core and 25GB of RAM. Our algorithm implementation can be found on 
GitHub\footnote{{\small{\url{https://github.com/mengxianglgal/Quant-BnB}}}}.

\subsection{Comparison of different lower bounds}\label{sect:numexp1}

Recall that \alg~ requires a lower-bound function $W$ such that $W(\cI, \phi) \le L_2(\cI,\phi)$ for all $ \phi \in \Phi$ and $\cI \subseteq [n]$. In addition, the efficiency of the algorithm depends largely on the choice of $W$. We have developed 4 possible lower bounds---$W_0$, $W_{1,s'}$, $W_2$ and $L_2$, and theoretically studied the quality and computational cost of them (see \eqref{eq:lowerbound-qual} and \eqref{eq:lowerbound-time}). In this experiment, we figure out their practical performance.

We set the parameter $s$ in Algorithm \ref{alg:bb-depth2} to be $3$, and the parameter $s'$ is dynamically chosen as $\lfloor\frac{0.6ns}{b-a}\rfloor$ for tuple $\phi=(f_0,[a,b],f_1,f_2)$. As discussed in Section \ref{sect:correctness}, the computational cost of $W_0$ and $W_{1,s'}$ is linear w.r.t $n$ under such setting, while computing $W_{2}$ and $L_{2}$ in the first iteration of \alg~costs $O(n^2)$.

We implement \alg~ with the lower-bound function $W$ in Step 1-2 chosen as $W_0,\,W_{1,s'},\,W_2$ and $L_2$, respectively. Table \ref{tab:differentlb} displays the computation time of these four choices on UCI datasets with number of data points $n\ge 10^4$. Although $W_2$ produces a tighter lower bound compared to $W_0$ and $W_{1,s'}$ (which helps \alg~prune more trees in every iteration), taking $W=W_2$ still has a bad performance due to its expensive computational cost. In contrast, proposed lower bounds $W_0$ and $W_{1,s'}$ result in significant speedups compared to computing $L_2$ directly. In the following experiments, we always choose $W=W_{1,s'}$ to reduce computational cost.

 \begin{table}[H]
\small
\begin{adjustbox}{width=0.55\columnwidth,center}
\begin{tabular}{ll|cccc} \Xhline{3\arrayrulewidth}
Name    &(n,p)      & $W_0$ & $W_{1,s'}$ & $W_2$ & $L_2$  \\ \hline
avila   &(10432,10)  & 5.1   & 4.5        & -     & 519   \\ 
bean    &(10888,16)  & 3.0   & 3.4        & -     & -     \\
eeg     &(11984,14)  & 2.9   & 2.9        & -   & 47    \\
htru    &(14318,8)   & 1.4   & 1.3        & 244   & 515   \\
magic   &(15216,10)  & 1.2   & 1.0        & -     & -     \\
skin    &(196045,3)  & 2.0   & 2.1        & -     & 31    \\
casp    &(36584,9)   & 6.7   & 4.2        & -     & -     \\
energy  &(15788,28)  & 18    & 14         & -     & -     \\
gas     &(29386,10)  & 1.5   & 1.5        & -     & 444   \\
news    &(31715,59)  & 301   & 349        & -     & -     \\
query2   &(159874,4)  & 15    & 9.8        & -     & -     \\
\Xhline{3\arrayrulewidth}
\end{tabular}
\end{adjustbox}
\caption{\small{\alg~with four different methods for lower bound computation on depth-2 trees. For each dataset, the number of observations and the number of features are provided. Each entry denotes running time in seconds. Symbol `-' refers to time out (10min).  }}
\label{tab:differentlb}
\end{table}

\subsection{Comparison with state-of-the-art optimal methods}\label{sect:numexp2}

We compare our algorithm with the recently proposed methods for solving optimal classification trees: BinOCT \cite{verwer2019learning}, MurTree \cite{JMLR:v23:20-520} and DL8.5 \cite{aglin2020learning}.  We also tested other competing algorithms, but they all took substantially longer time to deliver optimal trees---see Appendix for details.
Since MurTree and DL8.5 apply to datasets with binary features, we adopt the equivalent-conversion pre-processing used in \cite{lin2020generalized} by encoding each continuous feature $f$ to a set of $u(f)-1$ binary features, using all possible thresholds.

The computation time for learning optimal shallow trees (depth $=2$, $3$) on classification tasks is presented in Table~\ref{tab:class}. 
\texttt{Quant-BnB} can solve depth-$2$ trees in a few seconds--a speedup of several orders of magnitude compared to other methods. When the depth is $3$, \texttt{Quant-BnB} still outperforms competing algorithms by a large margin on 13 of 16 datasets. Although being highly effective on datasets with purely binary features, MurTree and DL8.5 
can be expensive to deliver optimal trees on datasets with continuous features. 
This is perhaps due to the increase in number of features while converting the continuous features to binary features. On a related note, it is worth pointing out that one may use approximate methods to convert continuous to binary 
features~\cite{JMLR:v23:20-520}---however, such approximate schemes may result in a lossy compression of the training data as shown in our experiments in the Appendix. Note also that we do not use any compression of features for \alg. The numerical results illustrate the effectiveness of \alg~in solving shallow optimal trees on large datasets with continuous features.

\begin{table}[!b]
\begin{adjustbox}{width=0.98\columnwidth,center}
\scalebox{0.9}{\begin{tabular}{ll|cccc|cccc} \Xhline{3\arrayrulewidth}
Dataset
& \multirow{2}{*}{(n,p)} & \multicolumn{4}{c|}{depth=2}        & \multicolumn{4}{c}{depth=3}        \\
Name                      &                    & Quant-BnB &BinOCT& MurTree & DL8.5 & Quant-BnB &BinOCT& MurTree & DL8.5  \\ \hline
avila&(10430,10)&\textbf{ 4.5 }&(1.3\%)&OoM&3278 &\textbf{ 4188 }&OoM&OoM&OoM\\
bank&(1097,4)&\textbf{ \textless 0.1 }&2963 &8.4 &4.6 &\textbf{ 4.4 }&(100\%)&142 &-\\
bean&(10888,16)&\textbf{ 3.4 }&(0\%)&OoM&OoM&\textbf{ 1014 }&OoM&OoM&OoM\\
bidding&(5056,9)&\textbf{ 0.2 }&(1.1\%)&345 &72 &\textbf{ 30 }&(154\%)&6252 &OoM\\
eeg&(11984,14)&\textbf{ 2.9 }&(6.3\%)&288 &34 &4042 &(13\%)&\textbf{ 1783 }&OoM\\
fault&(1552,27)&\textbf{ 1.6 }&(9.1\%)&530 &271 &-&(34\%)&-&OoM\\
htru&(14318,8)&\textbf{ 1.3 }&(7.7\%)&OoM&OoM&\textbf{ 10303 }&OoM&OoM&OoM\\
magic&(15216,10)&\textbf{ 1.0 }&(2.6\%)&OoM&OoM&\textbf{ 1090 }&(14\%)&OoM&OoM\\
occupancy&(8143,5)&\textbf{ 0.3 }&(2.3\%)&193 &33 &\textbf{ 106 }&(28\%)&1692 &OoM\\
page&(4378,10)&\textbf{ 0.4 }&(0\%)&155 &84 &\textbf{ 471 }&(35\%)&-&OoM\\
raisin&(720,7)&\textbf{ 0.1 }&9590 &13 &6.2 &\textbf{ 167 }&(6.6\%)&432 &-\\
rice&(3048,7)&\textbf{ 0.4 }&(3.0\%)&591 &267 &\textbf{ 1340 }&(11\%)&-&OoM\\
room&(8103,16)&\textbf{ 1.0 }&(25\%)&18 &14 &\textbf{ 180 }&(239\%)&269 &-\\
segment&(1848,18)&\textbf{ 1.1 }&(0.5\%)&389 &213 &\textbf{ 153 }&(250\%)&-&OoM\\
skin&(196045,3)&\textbf{ 2.3 }&(28\%)&37 &16 & 350 &(9.9\%)&\textbf{112} &-\\
wilt&(4339,5)&\textbf{ 0.2 }&(0\%)&653 &314 &\textbf{ 67 }&(56\%)&-&OoM\\
\Xhline{3\arrayrulewidth}
\end{tabular}}
\end{adjustbox}
\caption{\small{Comparison of  \alg~aginst BinOCT, MurTree and DL8.5 on 16 classification datasets. For each dataset, the number of observations and the number of features are provided. Each entry denotes running time in seconds. - refers to time out (4h), OoM refers to out of memory (25GB).  If BinOCT times out, we display the relative difference $(L_B-L_Q)/L_Q$ as a percentage instead, where $L_B$ and $L_Q$ are the training errors of BinOCT and \texttt{Quant-BnB}, respectively.}}
\label{tab:class}
\end{table}

\subsection{Comparison with heuristic methods}\label{sect:numexp3}

We study the test-accuracy of optimal decision trees. 
Earlier work~\cite{verwer2019learning,JMLR:v23:20-520} in classification with binary features suggest that optimal trees can lead to better test-performance compared to heuristics.
We explore if similar empirical findings hold true for the tasks we consider herein.
We compare our approach with the well-known algorithm CART \cite{breiman1984classification} and Tree Alternating Optimization (TAO) \cite{carreira2018alternating}---both based on heuristics. Both CART, TAO consider the same models as \texttt{Quant-BnB}, namely, axis-aligned trees with depth $2$ or $3$.

We compare the test error on a collection of 27 datasets: 16 classification and 11 regression tasks (see Appendix for details). Since the range of loss function of each test set varies, we study the relative loss $(L^{te}_C-L^{te}_Q)/L^{te}_Q$, where $L^{te}_Q$ is test error of \alg, and  $L^{te}_C$ is the test error of the competing algorithm (here, $C$ is CART or TAO).
The results are summarized in Fig \ref{fig:test}. We observe from the figure that \texttt{Quant-BnB} obtains depth$-2$ trees with lower test error in more than $66\%$ datasets. When the depth is $3$, \alg~leads to better generalization in most of cases compared to CART and TAO. The prediction performance of TAO is slightly better than CART, at the expense of higher computational cost. The results indicate that optimal trees delivered by \texttt{Quant-BnB} offer an edge compared to heuristic methods, especially for deeper trees.

\begin{figure}[!b]
    \centering\vspace{-7mm}
    \subfloat[Comparison with CART, Depth$=2$\label{subfig-1}]{
      \includegraphics[width=0.5\textwidth]{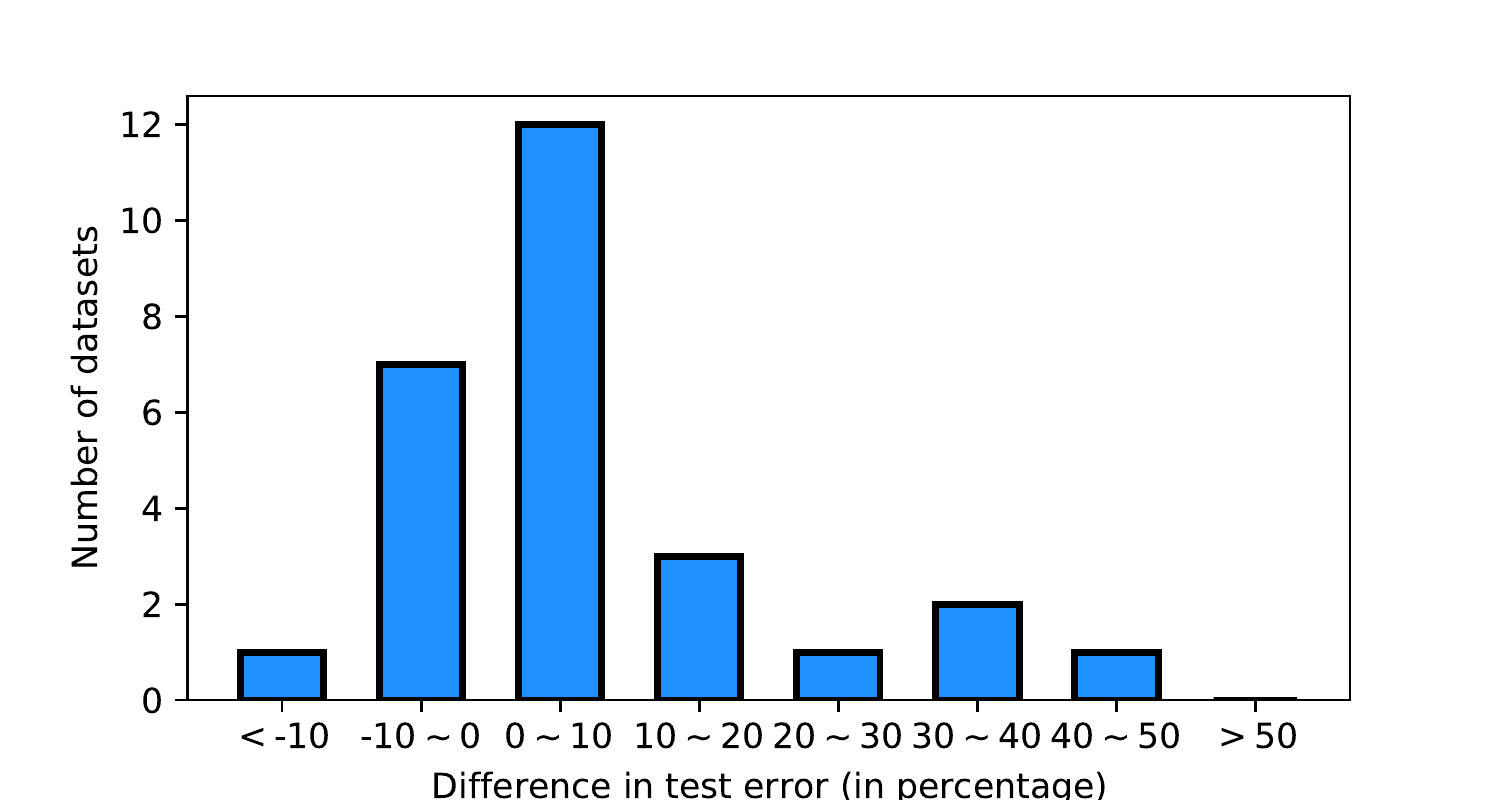}
    }
    \subfloat[Comparison with TAO, Depth$=2$\label{subfig-2}]{
       \includegraphics[width=0.5\textwidth]{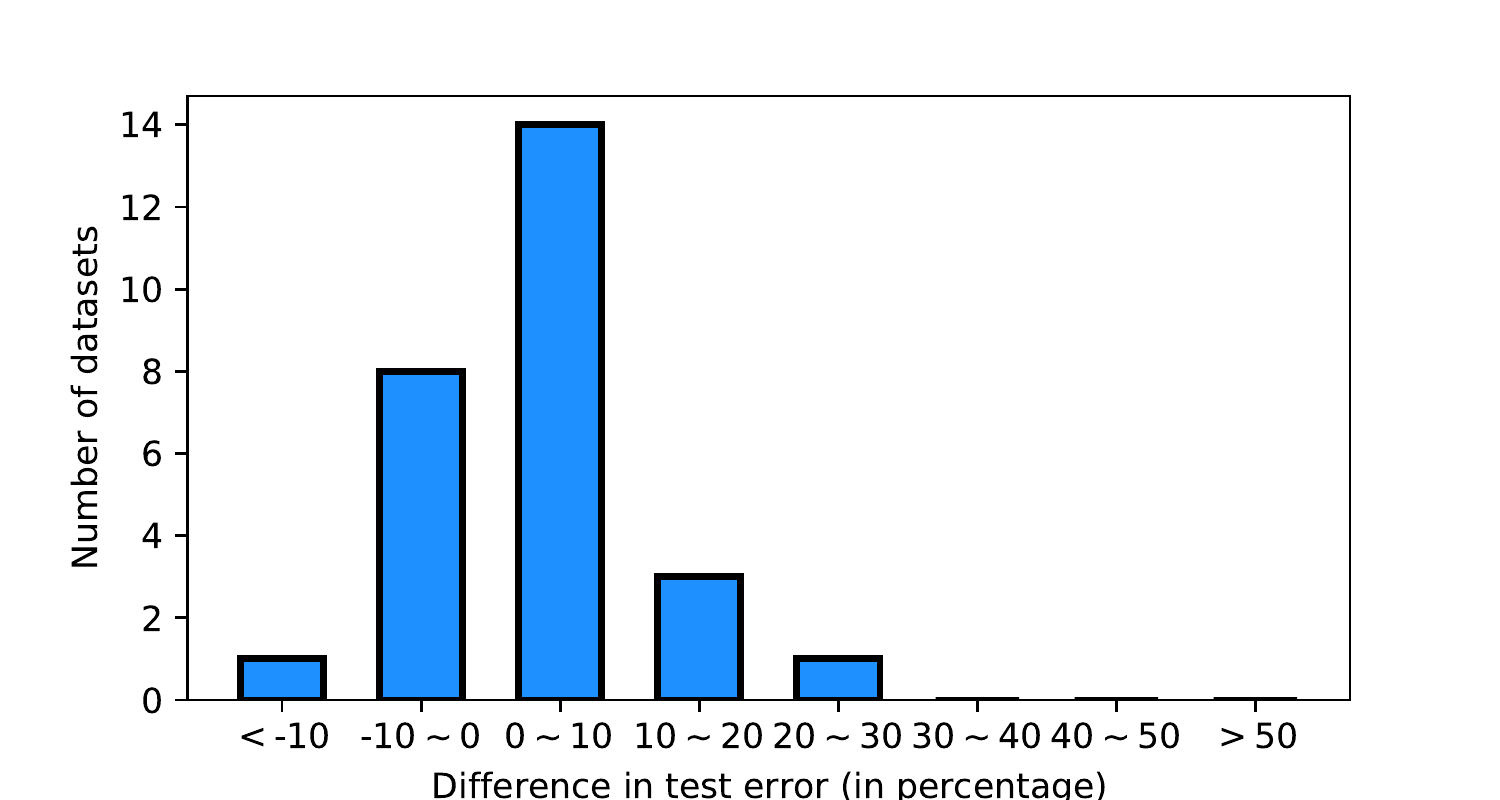}
    }\\ \vspace{-2mm}
    \subfloat[Comparison with CART, Depth$=3$\label{subfig-3}]{
      \includegraphics[width=0.5\textwidth]{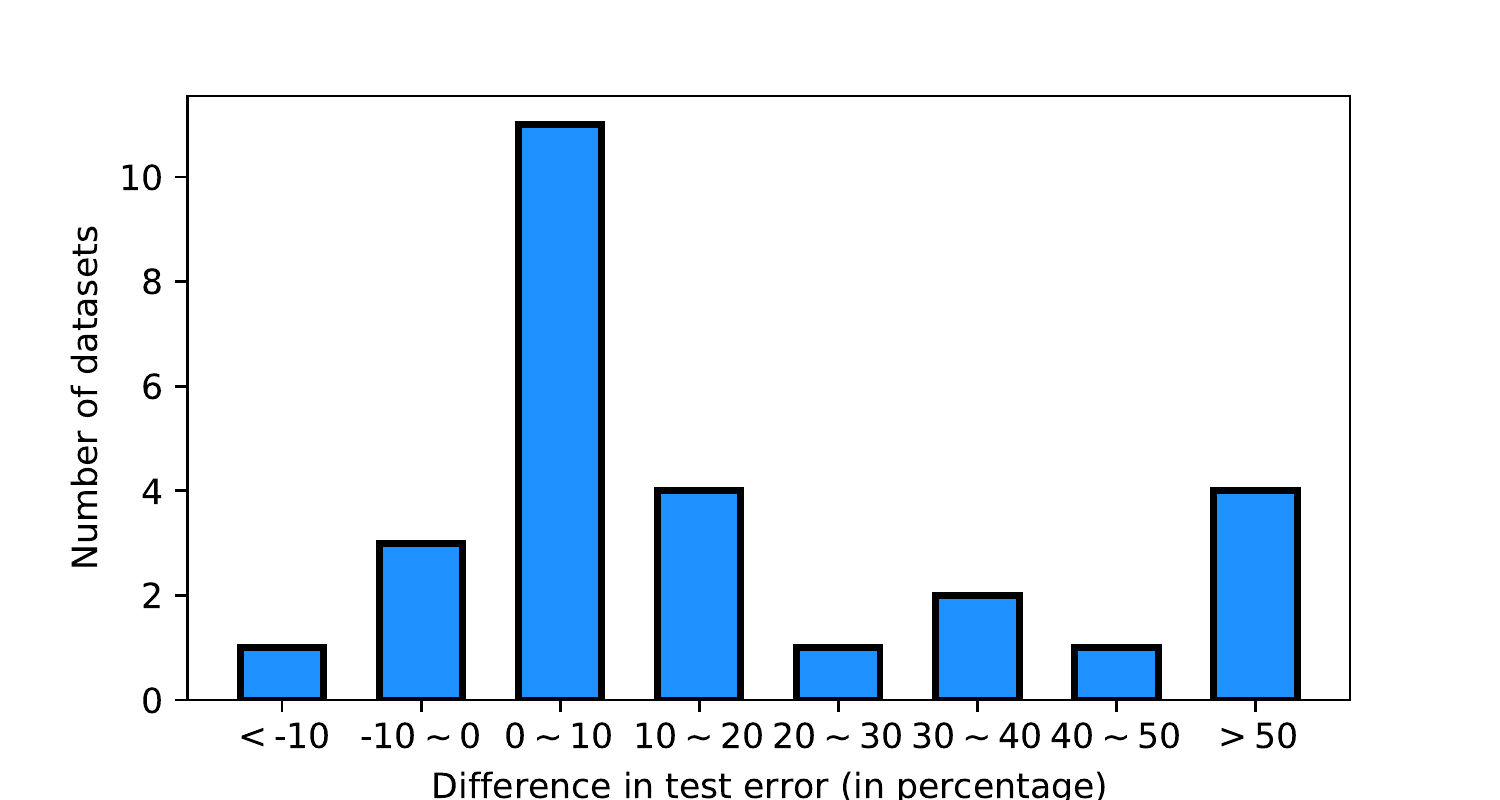}
    }
    \subfloat[Comparison with TAO, Depth$=3$\label{subfig-4}]{
       \includegraphics[width=0.5\textwidth]{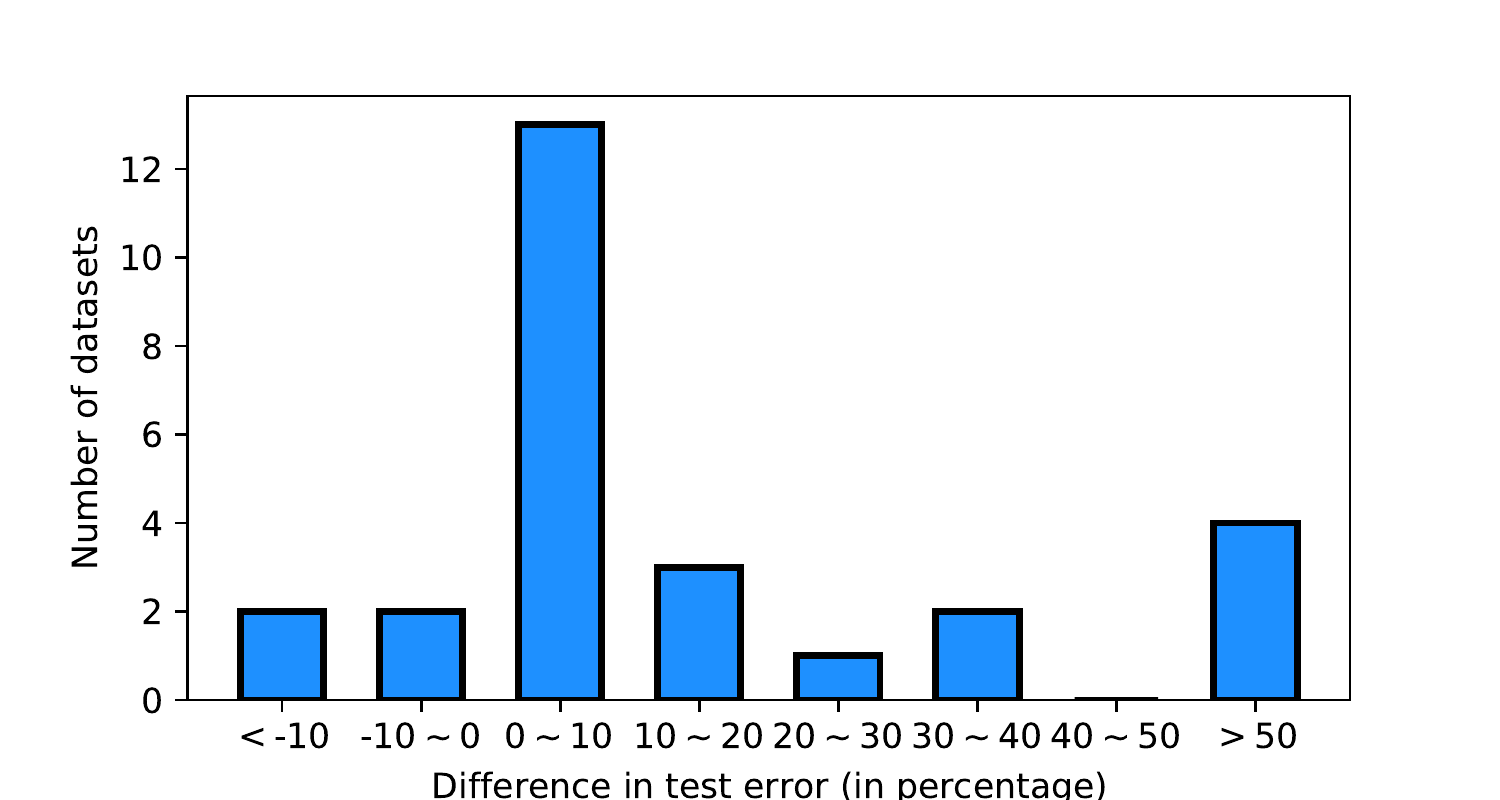}
    }\vspace{-2mm}
    \caption{\small{Performance comparison of \texttt{Quant-BnB} against CART  and TAO on 27 datasets. $L_C^{te}$ is test error of algorithm $C$ (i.e., CART or TAO) and $L_Q^{te}$ is test error of \texttt{Quant-BnB}. We summarize the 
    relative difference $(L_C^{te}-L_Q^{te})/L_Q^{te}$-values between
    a heuristic method (C) and optimal decision trees delivered by \texttt{Quant-BnB} (Q), shown in percentages using bar charts. A positive value of the x-axis means \alg~performs better than competing methods.}}\label{fig:test}
\end{figure}

\section{Conclusions and Discussions}

We present a novel BnB framework for optimal decision trees that applies to both regression and classification problems with continuous features.  This extend the scope of optimal procedures in the literature that have been developed for classification problems with binary features. Our approach is based on partitioning the feature values based on quantiles and using them to generate upper and lower bounds. We discuss convergence guarantees of our procedure. Numerical experiments suggest the 
efficiency of our approach for shallow decision trees.

Although a single optimal shallow tree appears to be somewhat restrictive in terms of prediction, it can be useful for interpretability (it can be difficult to interpret trees with depth much larger than 3). 
Additionally, a heuristic procedure (e.g., CART) may require a larger depth to achieve the same training/test error as
an optimal tree with $d=3$. To improve prediction performance of a single shallow tree, one can use an ensemble (e.g., random forest, Boosting) of shallow trees.

\section*{Acknowledgements} We thank the reviewers for their comments that resulted in improvements in the paper. 
This research is supported in part by grants from the Office of Naval Research (N00014-21-1-2841) and
Liberty Mutual Insurance.

\appendix
\section{Examples}

\subsection{An example of the space $\cT_2(f_0, [a,b], F_1, F_2)$}

We consider a classification dataset with $n = 6$, $p=3$. The feature vectors and labels are provided in Table \ref{tab:example}.
\begin{table}[htbp!]
\centering
\begin{tabular}{l|llllll} \Xhline{3\arrayrulewidth}
Data       & $x_1$ & $x_2$ & $x_3$ & $x_4$ & $x_5$ & $x_6$ \\ \hline
Feature 1 & 1 & 2 & 3 & 3 & 4 & 5 \\
Feature 2 & 0 & 1 & 2 & 3 & 4 & 5 \\
Feature 3 & 0 & 0 & 3 & 3 & 5 & 5 \\
Label     & 1 & 2 & 1 & 2 & 1 & 2 \\
\Xhline{3\arrayrulewidth}
\end{tabular}
\caption{An classification example with $6$ samples. Each has $3$ features and a label in $\{1,2\}$. }
\label{tab:example}.
\end{table}

Then we have $u(1) = 5$, and
\begin{equation}
    w_0^{1} = -\infty, 
    ~~  w_1^{1} = 1,
    ~~ w_{2}^{1} = 2, 
    ~~ w_{3}^{1} = 3, 
    ~~ w_{4}^{1} = 4, 
    ~~ w_{5}^{1} = 5, 
    ~~ w_{6}^{1} = \infty,
\end{equation}
\begin{equation}
    \td w_0^{1} = -\infty, 
    ~~ \td w_1^{1} = 1.5, 
    ~~ \td w_2^{1} = 2.5, 
    ~~ \td w_3^{1} = 3.5,
    ~~ \td w_4^{1} = 4.5,
    ~~ \td w_5^{1} = \infty.
\end{equation}
As an example, the set $\cT_2(1, [1,4], \{2,3\}, \{2,3\})$ contains all trees whose splitting features at the root node are $1$; splitting thresholds at the root node are in $\{ 1.5, ~ 2.5, ~3.5,~ 4.5 \}$; splitting features at the left child and the right child are in $\{2,3\}$. 

\subsection{An example of a single iteration of \texttt{Quant-BnB}}

We follow the assumption on data given in the previous section. Now suppose at some iteration $k$, the set $\AL^{(k-1)}$ contains a single tuple $(1, [1,4], \{2,3\}, \{2,3\})$. The current upper bound $U$ equals to $2$, and the parameter $s$ in Algorithm \ref{alg:bb-depth2} is set to be $2$. In addition, we choose $W_0(\cI, \phi)$ defined in Eq.\eqref{def:W0} as the lower bound required in Proposition \ref{prop: quantile-prune}.

To construct $\AL^{(k)}$, \alg~checks the tuple $(1, [1,4], \{2,3\}, \{2,3\})$. Since $s=2\le 4-1$, the algorithm computes $(t_0,t_1,t_2)=(1,2,4)$ being almost $2$-equi-spaced in $[1,4]$ and conducts the following 2 steps.

\begin{itemize}
    \item (Step1: Update upper bound) \texttt{Quant-BnB} computes $U'$ by
    \begin{align*}
     U'   &= \min_{f_1\in \{2,3\}, f_2\in \{2,3\}} \big\{
        V_2([6], 1, [1,4], f_1, f_2) \big\}   \\
&=\min_{f_1\in \{2,3\}, f_2\in \{2,3\}}\Bigg\{\min \Big\{L_1(\{1\},f_1)+L_1(\{2,3,4,5,6\},f_2), \\ 
&\qquad\qquad\qquad\qquad\qquad\quad L_1(\{1,2\},f_1)+L_1(\{3,4,5,6\},f_2),\,\,\,L_1(\{1,2,3,4,5\},f_1)+L_1(\{6\},f_2)\Big\} \Bigg\} \\ 
&=1.
    \end{align*}
    Since $U' < U$, \texttt{Quant-BnB} then updates $U=U'=1$.
    
    \item (Step2: Compute lower bound and prune) \texttt{Quant-BnB} computes $W_0([n],\phi^j_{f_1,f_2})$ for any $j\in \{1,2\}$, $f_1,\,f_2\in \{2,3\}$. The results are (computation details are omitted for simplicity)
    \begin{align*}
  W_0([n],\phi^1_{2,2}) = 1, ~ W_0([n],\phi^1_{2,3}) = 2  , ~W_0([n],\phi^1_{3,2}) = 1, ~ W_0([n],\phi^1_{3,3}) = 2,\\
  W_0([n],\phi^2_{2,2}) = 0, ~ W_0([n],\phi^2_{2,3}) = 1  , ~W_0([n],\phi^2_{3,2}) = 0, ~ W_0([n],\phi^2_{3,3}) = 1.
    \end{align*}
   The algorithm then computes sets $F_{1,j}$ and $F_{2,j}$ according to \eqref{eq:def-F1j} and \eqref{eq:def-F2j} 
   as
\begin{align*}
    F_{1,1} &:= \big\{f_1\in F_1 ~\big|~ \min_{f_2\in F_2} W([n],\phi^1_{f_1,f_2}) \le U  \big\}=\{2,3\},  \\
     F_{1,2} &:= \big\{f_1\in F_1 ~\big|~ \min_{f_2\in F_2} W([n],\phi^2_{f_1,f_2}) \le U  \big\}=\{2,3\},   \\
F_{2,1} &:= \big\{f_2\in F_2 ~\big|~ \min_{f_1\in F_1} W([n],\phi^1_{f_1,f_2}) \le U  \big\}=\{2\},  \\
F_{2,2} &:= \big\{f_2\in F_2 ~\big|~ \min_{f_1\in F_1} W([n],\phi^2_{f_1,f_2}) \le U  \big\}=\{2,3\}.
\end{align*}
    Finally, \texttt{Quant-BnB} updates 
    \begin{equation}
        \AL^{(k)} =   \{(1, [1, 2], F_{1,1}, F_{2,1})\} \cup \{(1, [2, 4], F_{1,2}, F_{2,2})\}. \nonumber
    \end{equation}
    
\end{itemize}

\section{Appendix for proofs}

\subsection{Auxiliary results}

We first prove a basic equality for $L_0$. 

\begin{lemma}\label{lemma:var}
For any disjoint sets $\cI,\cJ \subseteq [n]$, it holds 
\begin{equation}
   L_0(\cI \cup \cJ) \ge
L_0(\cI)+L_0(\cJ).
\end{equation}
\end{lemma}

\begin{proof}
Note that 
\begin{equation}
\begin{aligned}
 L_0(\cI \cup \cJ) & \ = \ \min_{y\in \cY} \sum_{i\in \cI \cup \cJ} \ell(y_i, y) \ = \ 
 \min_{y\in \cY} \Big\{  \sum_{i\in \cI} \ell(y_i, y) + \sum_{j\in \cJ} \ell(y_j, y)  \Big\} \\
 & \ \ge \ \min_{y\in \cY} \Big\{ \sum_{i\in \cI} \ell(y_i, y)\Big\} + \min_{y\in \cY} \Big\{ \sum_{j\in \cJ} \ell(y_j, y)\Big\}  \ = \ 
 L_0(\cI) + L_0(\cJ).
\end{aligned} 
\end{equation}

\end{proof}

Using the equality above, we prove a useful inequality for $L_1$ presented below.

\begin{lemma}\label{lemma:L1-super-add}
For any disjoint sets $\cI, \cJ \subseteq [n]$ and any $f\in [p]$, it holds 
\begin{equation}
    L_1(\cI\cup \cJ, f) \ge L_1(\cI , f) + L_1(\cJ , f).
\end{equation}
\end{lemma}

\begin{proof}
By the definition \ref{eq:defL1}, there exists integer $t^*$ with $0 \le t^* \le u(f)$ such that 
\begin{equation}\label{ineq1}
    L_1(\cI \cup \cJ, f) = 
    L_0 \Big( (\cI\cup \cJ)_{[0,t^*]}^f  \Big) + 
    L_0 \Big( (\cI\cup \cJ)_{[t^*, u(f)]}^f  \Big).
\end{equation}
Note that $(\cI\cup \cJ)_{[0,t^*]}^f = (\cI)_{[0,t^*]}^f \cup ( \cJ)_{[0,t^*]}^f$ and $(\cI)_{[0,t^*]}^f \cap ( \cJ)_{[0,t^*]}^f = \emptyset$, so by Lemma \ref{lemma:var}, we have 
\begin{equation}\label{ineq2}
    L_0 \Big( (\cI\cup \cJ)_{[0,t^*]}^f  \Big) \ge L_0(\cI_{[0,t^*]}^f  )
    + L_0(\cJ_{[0,t^*]}^f  ).
\end{equation}
By a similar argument we have
\begin{equation}\label{ineq3}
    L_0 \Big( (\cI\cup \cJ)_{[t^*, u(f)]}^f  \Big) \ge L_0(\cI_{[t^*, u(f)]}^f  )
    + L_0(\cJ_{[t^*,u(f)]}^f  ).
\end{equation}
By \eqref{ineq1}, \eqref{ineq2} and \eqref{ineq3} we have
\begin{equation}
\begin{aligned}
L_1(\cI \cup \cJ, f) &~\ge~ 
L_0(\cI_{[0,t^*]}^f  )
     +L_0(\cI_{[t^*, u(f)]}^f  )
    + L_0(\cJ_{[0,t^*]}^f  )
    + L_0(\cJ_{[t^*,u(f)]}^f  ) \\
    &~\ge~
    \min_{0\le t \le u(f)} \big\{
    L_0(\cI_{[0,t]}^f  )
     +L_0(\cI_{[t, u(f)]}^f  )
    \big\} + 
    \min_{0\le t \le u(f)} \big\{
    L_0(\cJ_{[0,t]}^f  )
    + L_0(\cJ_{[t,u(f)]}^f  )
    \big\} \\
    &~=~ L_1(\cI, f) + L_1(\cJ, f) .
\end{aligned}
\nonumber
\end{equation}
This completes the proof.
\end{proof}

A similar inequality also holds for $L_2$, as shown below. 

\begin{lemma}\label{lemma:L2-super-add}
For any disjoint sets $\cI, \cJ \subseteq [n]$ and any $\phi\in \Phi$, it holds 
\begin{equation}
    L_2(\cI\cup \cJ, \phi) \ge L_2(\cI , \phi) + L_2(\cJ , \phi). \nonumber
\end{equation}
\end{lemma}
\begin{proof}
Given $\phi = (f_0, [a,b], f_1, f_2) \in \Phi$, 
from the equality \eqref{def:L2}, there exists an integer $t^* \in [a,b]$ such that 
\begin{equation}\label{ineq-new1}
    L_2 (\cI\cup \cJ, \phi) = 
 L_1((\cI\cup\cJ)^{f_0}_{[0,t^*]},f_1) +  L_1((\cI\cup\cJ)^{f_0}_{[t^*,u({f_0})]},f_2).
\end{equation}
Note that $(\cI\cup \cJ)_{[0,t^*]}^{f_0} = (\cI)_{[0,t^*]}^{f_0} \cup ( \cJ)_{[0,t^*]}^{f_0}$ and $(\cI)_{[0,t^*]}^{f_0} \cap ( \cJ)_{[0,t^*]}^{f_0} = \emptyset$, so by Lemma \ref{lemma:L1-super-add}, we have 
\begin{equation}\label{ineq-new2}
    L_1 \Big( (\cI\cup \cJ)_{[0,t^*]}^{f_0}, f_1  \Big) \ge L_1(\cI_{[0,t^*]}^{f_0} , f_1 )
    + L_1(\cJ_{[0,t^*]}^{f_0}, f_1  ).
\end{equation}
By a similar argument we have
\begin{equation}\label{ineq-new3}
    L_1 \Big( (\cI\cup \cJ)_{[t^*, u(f_0)]}^{f_0}, f_2  \Big) \ge L_1(\cI_{[t^*, u(f_0)]}^{f_0} , f_2 )
    + L_1(\cJ_{[t^*,u(f_0)]}^{f_0}, f_2  ).
\end{equation}
Combining \eqref{ineq-new1}, \eqref{ineq-new2} and \eqref{ineq-new3}, we have 
\begin{equation}
\begin{aligned}
L_2(\cI \cup \cJ, \phi) &~\ge~ 
L_1(\cI_{[0,t^*]}^{f_0}, f_1  )
     +L_1(\cI_{[t^*, u(f_0)]}^{f_0}, f_2  )
    + L_1(\cJ_{[0,t^*]}^{f_0}, f_1  )
    + L_1(\cJ_{[t^*,u(f_0)]}^{f_0}, f_2  ) \\
    &~\ge~
    \min_{0\le t \le u(f_0)} \big\{
    L_1(\cI_{[0,t]}^{f_0}, f_1  )
     +L_1(\cI_{[t, u(f_0)]}^{f_0}, f_2  )
    \big\} + 
    \min_{0\le t \le u(f_0)} \big\{
    L_1(\cJ_{[0,t]}^{f_0} ,f_1 )
    + L_1(\cJ_{[t,u(f_0)]}^{f_0} , f_2 )
    \big\} \\
    &~=~ L_2(\cI, \phi) + L_2(\cJ, \phi).
\end{aligned}
\nonumber
\end{equation}
This completes the proof.

\end{proof}

\subsection{Proof of Lemma \ref{lemma:L2-lb}}
\begin{proof}
By definition \eqref{def:L2} we have \begin{equation}
L_2(\cI, f_0, [a,b], f_1, f_2) = 
\min_{a \le t \le b} \big\{  L_1(\cI^{f_0}_{[0,t]},f_1) +  L_1(\cI^{f_0}_{[t,u({f_0})]},f_2) \big\} .\nonumber
\end{equation}
Let $t^* $ be the integer in $[a,b]$ such that 
\begin{equation}\label{tmp1}
    L_2(\cI, f_0, [a,b], f_1, f_2) = 
    L_1(\cI^{f_0}_{[0,t^*]},f_1) +  L_1(\cI^{f_0}_{[t^*,u({f_0})]},f_2).
\end{equation}
Since there exists $j^*\in [s']$ such that $t_{j^*-1} \le t^* \le t_{j^*} $, by Lemma \ref{lemma:L1-super-add} we have
\begin{equation}\label{tmp2}
L_1(\cI^{f_0}_{[0,t^*]},f_1) \ge 
L_1(\cI^{f_0}_{[0,t_{j^*-1}]},f_1), \quad \text{and}\quad 
L_1(\cI^{f_0}_{[t^*,u({f_0})]},f_2) \ge 
L_1(\cI^{f_0}_{[t_{j^*}, u(f_0)]},f_2). 
\end{equation}
Combining \eqref{tmp1} and \eqref{tmp2} we have
\begin{equation}
\begin{aligned}
L_2(\cI, f_0, [a,b], f_1, f_2) &\ge 
L_1(\cI^{f_0}_{[0,t_{j^*-1}]},f_1)+
L_1(\cI^{f_0}_{[t_{j^*}, u(f_0)]},f_2) \\
&\ge
\min_{j\in[s']} \big\{L_1(\cI^{f_0}_{[0,t_{j-1}]},f_1)+
L_1(\cI^{f_0}_{[t_{j}, u(f_0)]},f_2) \big\} \\
&= \widehat L_2 (\cI, f_0, [a,b], f_1, f_2, s').
\end{aligned}
\nonumber
\end{equation}
\end{proof}

\subsection{Proof of Lemma \ref{lemma:key-ineq}}

\begin{proof}
By definition it is trivial that $W_0(\cI, \phi) \le W_{1,s'} (\cI, \phi)$ for any $s'\le b-a$; 
By Lemma \ref{lemma:L2-lb}, we know that $W_{1,s'} (\cI, \phi) \le W_2(\cI, \phi)$. In the following, we prove that 
$ W_2(\cI, \phi) \le L_2(\cI,  \phi)$. 
Suppose $\phi = (f_0, [a,b], f_1, f_2)$ with $f_0,f_1,f_2 \in [p]$ and $0\le a \le b \le u(f_0)$. 
Note that
\begin{equation}
    \begin{aligned}
    L_2(\cI,  \phi) 
    &= \min_{t\in [a,b]} 
    \Big\{ L_1( \cI_{[0,t]}^{f_0} , f_1 ) 
    + L_1(\cI_{[t,u(f_0)]}^{f_0}, f_2 )\Big\} \\
    &\ge 
    \min_{t\in [a,b]} 
    \Big\{ L_1( \cI_{[0,a]}^{f_0} , f_1 ) +
    L_1( \cI_{[a,t]}^{f_0} , f_1 ) +
    L_1(\cI_{[t,b]}^{f_0}, f_2 ) + 
    L_1(\cI_{[b,u(f_0)]}^{f_0}, f_2 )
    \Big\} \\
    &= 
    L_1( \cI_{[0,a]}^{f_0} , f_1 ) +
    L_1( \cI_{[b,u(f_0)]}^{f_0} , f_2 )
    + 
    L_2( \cI^{f_0}_{[a,b]}, f_0, [a,b], f_1,f_2 ) ~ = ~ W_2(\cI, \phi),
    \end{aligned}
    \nonumber
\end{equation}
where the inequality follows from Lemma \ref{lemma:L1-super-add}. 
\end{proof}

\subsection{Proof of Proposition \ref{prop: quantile-prune}}
\begin{proof}
Note that 
\begin{equation}
    \cT_2(f_0, [a,b], F_1, F_2) = \bigcup_{j=1}^s \cT_2(f_0, [t_{j-1}, t_j], F_{1}, F_{2}).
\end{equation}
So we have 
\begin{equation}
   \cT_2(f_0, [a,b], F_1, F_2) \setminus \bigcup_{j=1}^s \cT_2(f_0, [t_{j-1}, t_j], F_{1,j}, F_{2,j}) = 
   \bigcup_{j=1}^s \Big( \cT_2(f_0, [t_{j-1},t_j], F_1, F_2) \setminus \cT_2(f_0, [t_{j-1}, t_j], F_{1,j}, F_{2,j}) \Big). \nonumber
\end{equation}
Suppose (for contradiction) that an optimal solution $T^*$ is in the l.h.s. of the above set, then there exists $j\in [s]$ such that 
\begin{equation}
T^* \in \cT_2(f_0, [t_{j-1},t_j], F_1, F_2) \setminus \cT_2(f_0, [t_{j-1}, t_j], F_{1,j}, F_{2,j}).
\end{equation}
Then we know $f_{\rootsub}(T^*) = f_0$, $f_{\lsub}(T^*) \in F_1$, $f_{\rsub}(T^*) \in F_2$,
and at least one of the following two cases hold: 
\begin{equation}
    (i) ~ f_{\lsub}(T^*) \in F_1 \setminus F_{1,j};
    \qquad 
    (ii) ~ f_{\rsub}(T^*) \in F_2 \setminus F_{2,j} .
\end{equation}
If $(i)$ holds, then we have 
\begin{equation}
    \begin{aligned}
     L_2([n], f_0, [t_{j-1}, t_j], f_{\lsub}(T^*), f_{\rsub}(T^*))
    &~\ge~ 
    W ([n], f_0, [t_{j-1}, t_j], f_{\lsub}(T^*), f_{\rsub}(T^*)) 
    \\
    &~\ge~
    \min_{f_2\in F_2} \big\{
    W ( [n], f_0, [t_{j-1}, t_j], f_{\lsub}(T^*), f_2) \big\} ~>~ U,
    \end{aligned}
\end{equation}
where the first inequality is by the assumption that $W(\cI,\phi) \le L_2(\cI, \phi)$ for all $\phi \in \Phi$; the last inequality is by the definition of $F_{1,j}$. 
Note that $L_2([n], f_0, [t_{j-1}, t_j], f_{\lsub}(T^*), f_{\rsub}(T^*))$ is the optimal value of \eqref{eq:opt}, this is a contradiction to the assumption that $U$ is an upper bound of the optimal value of \eqref{eq:opt}. If (ii) holds, by a similar argument as shown above we have a contradiction. 
 
\end{proof}

\subsection{Proof of Theorem \ref{thm:convergence}}
\begin{proof}
Define 
\begin{equation*}
    C_k:= \max_{(f_0,[a,b],F_1,F_2) \in \AL^{(k)}} \left\{b-a\right\}
\end{equation*}
as the length of the longest interval over all tuples in $\AL^{(k)}$. As stated in Section \ref{sect:overall strategy}, in each iteration Algorithm \ref{alg:bb-depth2} splits space $\cT_2( f_0, [a,b], F_1, F_2 )$ into at most $s$ subsets \begin{equation*}
    \bigcup_{j\in[s]}\cT_2( f_0, [t_{j-1},t_{j}],  F_{1,j},  F_{2,j} ).
\end{equation*}
Using the definition of $s$-equi-spaced points in $[a,b]$, and note that for any real numbers $s_1,s_2$ it holds $\lfl s_1 \rfl - \lfl s_2 \rfl \le \lceil s_1-s_2 \rceil$, we have
\begin{equation*}
    t_{j}-t_{j-1}\le  \left\lceil a+(j/s)(b-a)- a-((j-1)/s)(b-a)\right\rceil  =\left\lceil\frac{b-a}{s}\right\rceil\quad \text{ for all } j\in [s].
\end{equation*}
Namely, the length of the interval $[a,b]$ in the tuple $(f_0, [a,b], F_1, F_2 )$ reduces to $1/s$ of its original value in every iteration. Therefore, $C_k\le \left\lceil\frac{C_{k-1}}{s}\right\rceil$ for any $k\ge 1$.
Applying this inequality recursively, it holds that for any positive integers $k_0$, $m$, $C_{k_0}\le m$ as long as $C_0\le ms^{k_0}$. Setting $k_0=\lceil log_{s}(n)\rceil-1$ and $m =  \lceil\frac{C_0}{s^{k_0}}\rceil$ yields
\begin{equation*}
    C_{k_0} \le \Big\lceil\frac{C_0}{s^{k_0}}\Big\rceil\le\Big\lceil\frac{n}{s^{k_0}}\Big\rceil\le s.
\end{equation*}
If $\AL^{(k_0)}$ is not empty, then any tuple $(f_0, [a,b], F_1, F_2 )$ in it satisfies $b-a\le s$. The algorithm will then perform the exhaustive search method to each tuple in $\AL^{(k_0)}$ in the iteration $k_0+1$. Hence, $\AL^{(k_0+1)}=\emptyset$, and the algorithm terminates in at most $k_0+1=\lceil log_{s}(n)\rceil$ iterations.

Now, we prove that the algorithm yields an optimal solution. Suppose (by contradiction) that for some $k_0$, $\AL^{(k_0)}$ becomes empty and the algorithm offers an sub-optimal solution. This indicates that in some iteration $k\le k_0$, (at least one of if there exist multiple optimal solutions) the optimal solution is discarded during Step2, otherwise the algorithm would examine the loss of the optimal solution and record it since it is optimal. This contradicts to Proposition \ref{prop: quantile-prune}, which guarantees that no optimal solution will be eliminated in Step2. Therefore, the algorithm will give the optimal solution. The proof is completed.

\end{proof}

\subsection{Proof of Lemma \ref{lemma: compare-cost}}

\begin{proof}
Note that for all these three choices of $W$, Step 1 is the same. 

For the cost of Step 1, note that 
\begin{equation}
\begin{aligned}
U' &~=~  
    \min_{f_1\in F_1, f_2\in F_2, 0\le j\le s} \big\{
        L_1([n]_{[0,t_j]}^{f_0}, f_1) + L_1( [n]_{[t_j, u(f_0)]}^{f_0}, f_2 ) 
        \big\} \\
    &~=~
    \min_{0\le j\le s} \Big\{ 
    \min_{f_1\in F_1} \{ L_1([n]_{[0,t_j]}^{f_0}, f_1) \} + 
    \min_{f_2\in F_2} \{
    L_1( [n]_{[t_j, u(f_0)]}^{f_0}, f_2 ) 
    \}
    \Big\}.
\end{aligned}
\end{equation}
Recall that $L_1 (\cI, f) $ can be computed within $\wtd O(|\cI|)$ operations, so $U'$ can be computed within $\wtd O(n\td p s)$ operations.

The major cost of Step 2 lies in the computation of 
\begin{equation}\label{compute-W-1}
    \min_{f_2\in F_2} W([n],\phi^j_{f_1,f_2}) = \min_{f_2\in F_2} \big\{ W([n], f_0, [t_{j-1}, t_j], f_1, f_2)   \big\} \quad \text{ for all } f_1\in F_1 \text{ and } j\in [s], 
\end{equation}
and 
\begin{equation}\label{compute-W-2}
    \min_{f_1\in F_1} W([n],\phi^j_{f_1,f_2}) = \min_{f_1\in F_1} \big\{ W([n], f_0, [t_{j-1}, t_j], f_1, f_2)   \big\} 
    \quad \text{ for all } f_2\in F_2 \text{ and } j\in [s]. 
\end{equation}
Once \eqref{compute-W-1} and \eqref{compute-W-2} have been computed, the remaining cost of Step 2 can be bounded by $ O(\td p s)$. Below we show the costs of \eqref{compute-W-1} and \eqref{compute-W-2}, under different choices of $W$. 

(1) If $W = W_0$, we have 
\begin{equation}
    \begin{aligned}
    \min_{f_2\in F_2} \big\{ W([n], f_0, [t_{j-1}, t_j], f_1, f_2)   \big\}
    &= 
    \min_{f_2\in F_2} \Big\{   L_1([n]_{[0,t_{j-1}]}^{f_0}, f_1)  +
    L_1( [n]_{[t_j,n]}^{f_0} , f_2  ) 
    \Big\} \\
    &= 
    L_1 ( [n]_{[0,t_{j-1}]}^{f_0}, f_1 ) + 
    \min_{f_2\in F_2} \big\{  
    L_1( [n]_{[t_j,n]}^{f_0} , f_2  ) 
    \big\},
    \end{aligned}
\end{equation}
where the first equality makes use of the definition of $W_0$ and the assumption that $u(f_0) = n$. By the expression above, we know that computing \eqref{compute-W-1} requires at most $\wtd O(n\td ps)$ operations. 
By a similar argument, computing \eqref{compute-W-2} also requires at most $\wtd O(n\td ps)$ operations. Hence the cost of Step 2 is bounded by $\wtd O(n\td ps)$. 

(2) If $W = W_{1,s'}$ for some integer $s' \le b-a$, we have
\begin{equation}
    \begin{aligned}
    &\quad\min_{f_2\in F_2} \big\{ W([n], f_0, [t_{j-1}, t_j], f_1, f_2)   \big\} \\
    &= \min_{f_2 \in F_2} \Big\{ L_1([n]^{f_0}_{[0, t_{j-1}]}, f_1) 
+ L_1([n]^{f_0}_{[ t_j,n]}, f_2)
+ \widehat L_2([n]^{f_0}_{[ t_{j-1},t_j]}, f_0, [t_{j-1}, t_j], f_1,f_2, s') \Big\},
    \end{aligned}
\end{equation}
where we have used the assumption that $u(f_0) = n$. 
Let $(r_0, ...., r_{s'}) $ be the almost $s'$-equi-spaced integers in $[t_{j-1}, t_j]$, then by the definition of $\widehat L_2$ we have 
\begin{eqnarray}
&& \qquad\min_{f_2\in F_2} \big\{ W([n], f_0, [t_{j-1}, t_j], f_1, f_2)   \big\} \nonumber\\
&=&\min_{  f_2\in F_2, k\in [s'] } \Big\{  L_1( [n]_{[0,t_{j-1}]}^{f_0}, f_1 ) +
L_1 ( [n]^{f_0}_{[t_{j-1}, r_{k-1}]}, f_1 ) + 
L_1 ( [n]^{f_0}_{[r_{k}, t_j]}, f_2 ) + 
L_1 ( [n]^{f_0}_{[t_{j}, n]}, f_2 )\Big\}
\nonumber\\
&=& 
L_1( [n]_{[0,t_{j-1}]}^{f_0}, f_1 ) ~+~
\min_{  k\in [s'] } \Big\{   
L_1 ( [n]^{f_0}_{[t_{j-1}, r_{k-1}]}, f_1 ) 
+ 
\min_{f_2\in F_2} \big\{
L_1 ( [n]^{f_0}_{[r_{k}, t_j]}, f_2 ) + 
L_1 ( [n]^{f_0}_{[t_{j}, n]}, f_2 )
\big\}
\Big\} \nonumber\\
&:=&
J_1(j,f_1) ~+~
\min_{  k\in [s'] } \Big\{   
J_3(j, f_1, k)
+ 
\min_{f_2\in F_2} \big\{
J_4(j, f_2, k) + 
J_2(j, f_2)
\big\}
\Big\}. \label{eqq1}
\end{eqnarray}
where $J_1(j, f_1):= L_1( [n]_{[0,t_{j-1}]}^{f_0}, f_1 )$; 
$J_2(j, f_2):= L_1 ( [n]^{f_0}_{[t_{j}, n]}, f_2 )$; 
$J_3(j, f_1, k) :=  L_1 ( [n]^{f_0}_{[t_{j-1}, r_{k-1}]}, f_1 ) $
and 
$ J_4(j, f_2, k):= L_1 ( [n]^{f_0}_{[r_{k}, t_j]}, f_2 )$, and we highlighted the dependence on $j, f_1, f_2, k$. 
Note that:
\begin{itemize}
    \item $\{J_1(j, f_1) \}_{j\in [s], f_1\in F_1}$ can be computed with $\wtd O(n\td p s)$ operations. 
    \item $\{J_2(j, f_2)\}_{j\in [s], f_2\in F_2}$ can be computed with $\wtd O(n\td p s)$ operations. 
    \item $\{ J_3(j, f_1,k) \}_{j\in [s], f_1\in F_1, k\in [s']}$ can be computed with $\wtd O(\td p s' (t_j-t_{j-1})s) = \wtd O(\td p s' {(b-a)})$ operations. 
    \item $\{ J_4(j, f_2, k) \}_{j\in [s], f_2\in F_2, k\in [s']}$ can be computed with $\wtd O(\td p s' (t_j-t_{j-1})s) = \wtd O(\td p s' {(b-a)})$ operations. 
\end{itemize}
After the values above have been computed and maintained in memory, 
\begin{itemize}
    \item $ \Big\{\min_{f_2\in F_2} \big\{
J_4(j, f_2, k) + 
J_2(j, f_2)
\big\}
\Big\}_{j\in [s], k\in [s']}$ can be computed with $O(\td p ss')$ operations. 
\end{itemize}
Based on this, we know 
\begin{itemize}
    \item Computing 
    $\min_{  k\in [s'] } \Big\{   
J_3(j, f_1, k)
+ 
\min_{f_2\in F_2} \big\{
J_4(j, f_2, k) + 
J_2(j, f_2)
\big\}
\Big\} $ for all $j\in [s]$ and $f_1\in F_1$ requires at most $O(\td p s s')$ operations. 
\end{itemize}
With the analysis above, the computation of \eqref{compute-W-1}
requires at most $ \wtd O(n\td p s + \td p s' {(b-a)})$ operations. By a similar analysis, the computation of \eqref{compute-W-2} also
requires at most $ \wtd O(n\td p s + \td p s' {(b-a)})$ operations. 

(3) If $W = W_2$, the analysis is the same as the analysis for $W = W_{1,s'}$, with $s'$ being of the same order as $ (b-a)/s$.

(4) If $W = L_2$, it holds
\begin{equation}
    \begin{aligned}
     \min_{f\in F_2} \ W([n], f_0, [t_{j-1},t_j], f_1, f_2) 
    & ~=~ \min_{f\in F_2} \ 
    L_2([n], f_0, [t_{j-1},t_j], f_1, f_2) \\
    &~=~ 
    \min_{f_2\in F_2, t_{j-1} \le t \le t_j} 
    \Big\{ 
    L_1 ( [n]_{[0,t]}^{f_0}, f_1) + L_1([n]_{[t,n]}^{f_0}, f_2) 
    \Big\}.
    \end{aligned}
\end{equation}
For a fixed $j\in [s]$, computing the expression above requires $\wtd O((t_j - t_{j-1})\td p n)$ operations, so the computation of \eqref{compute-W-1} takes $\wtd O(n\td p (b-a))$ operations. By a similar analysis, the computation of \eqref{compute-W-2} also takes $\wtd O(n\td p (b-a))$ operations.

\end{proof}

\section{\texttt{Quant-BnB} for depth-3 optimal regression trees}\label{app-sect: depth-3}
In the following we present details on \texttt{Quant-BnB} for depth-3 optimal regression trees. 

\subsection{Notations and preliminaries}
Let $\cT_3$ be the set of all decision trees with depth $3$ whose splitting thresholds are in $ \{ \td w^f_{t} \}_{f\in [p], 0\le t \le u(f) }$. The problem of optimal regression tree with depth $3$ can be formulated as 
\begin{equation}\label{eq:opt-d=3}
\min_{T\in \cT_3} \sum_{i=1}^n \ell (y_i, T(x_i)).
\end{equation}
We use the notations shown in Figure \ref{fig:notation-depth3} to denote the nodes in a tree $T \in \cT_3$. 
For $S \in \{O,L,R, LL, LR, RL, RR \}$, let 
$(f_S(T), t_S(T)) $ denote the splitting rule for $ N_{S}(T)$. 

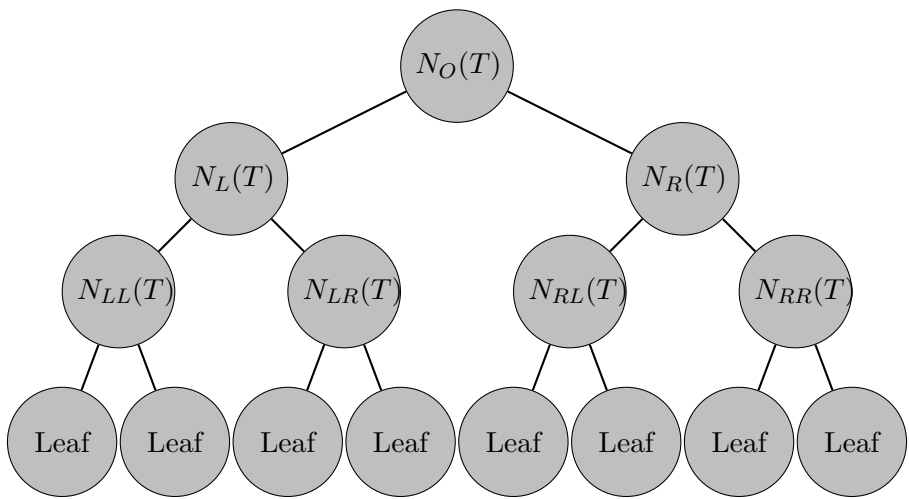
\begin{figure}[h]
\centering
\begin{tikzpicture}
   \tikzstyle{treenode} = [circle, minimum width=1.1cm, minimum height=1.1cm, text centered, text width=1.1cm, draw=black,  fill=gray!50]
   \tikzstyle{tedge} = [thick,-,>=stealth]

   \node[treenode] (n1) at (0,0.5) {$N_{\rootsub}(T)$};
   \node[treenode] (n2) at (-3,-1) {$N_{\lsub}(T)$};
   \node[treenode] (n3) at (3,-1) {$N_{\rsub}(T)$};
   \node[treenode] (n4) at (-4.5,-2.5) {$N_{\llsub}(T)$};
   \node[treenode] (n5) at (-1.5,-2.5) {$N_{\lrsub}(T)$};
   \node[treenode] (n6) at (1.5,-2.5) {$N_{\rlsub}(T)$};
   \node[treenode] (n7) at (4.5,-2.5) {$N_{\rrsub}(T)$};
   \node[treenode] (l1) at (-5.25,-4.5) {Leaf};
   \node[treenode] (l2) at (-3.75,-4.5) {Leaf};
   \node[treenode] (l3) at (-2.25,-4.5) {Leaf};
   \node[treenode] (l4) at (-0.75,-4.5) {Leaf};
   \node[treenode] (l5) at (0.75,-4.5) {Leaf};
   \node[treenode] (l6) at (2.25,-4.5) {Leaf};
   \node[treenode] (l7) at (3.75,-4.5) {Leaf};
   \node[treenode] (l8) at (5.25,-4.5) {Leaf};

   \draw[tedge] (n1) -- (n2);
   \draw[tedge] (n1) -- (n3);
   \draw[tedge] (n2) -- (n4);
   \draw[tedge] (n2) -- (n5);
   \draw[tedge] (n3) -- (n6);
   \draw[tedge] (n3) -- (n7);
   \draw[tedge] (n4) -- (l1);
   \draw[tedge] (n4) -- (l2);
   \draw[tedge] (n5) -- (l3);
   \draw[tedge] (n5) -- (l4);
   \draw[tedge] (n6) -- (l5);
   \draw[tedge] (n6) -- (l6);
   \draw[tedge] (n7) -- (l7);
   \draw[tedge] (n7) -- (l8);
\end{tikzpicture} 
\caption{Decision tree with depth $3$.}
\label{fig:notation-depth3}
\end{figure}

Given $f_0\in [p]$, integers $a$ and $b$ with $0\le a \le b \le u(f_0)$ and $\phi_1, \phi_2 \in \Phi$ with 
$\phi_1 = (f_1, [a_1,b_1],f_{1,1}, f_{1,2})$ and $\phi_2 = (f_2, [a_2,b_2],f_{2,1}, f_{2,2})$, define 
\begin{equation}
    \cT_3 (f_0, [a_0, b_0], \phi_1, \phi_2)
\end{equation}
to be the set of all trees $T \in \cT_3$ satisfying: 
$f_{\rootsub}(T) = f_0$, $t_{\rootsub} \in [a_0, b_0]$; 
~$f_{\lsub}(T) = f_1$, $t_{\lsub}(T) \in [a_1,b_1]$, $ f_{\llsub}(T) = f_{1,1} $, $f_{\lrsub}(T) = f_{1,2}$;
~ $f_{\rsub} (T) = f_{2}$, $t_{\rsub}(T) \in [a_2,b_2]$, $f_{\rlsub}(T) = f_{2,1}$ and $ f_{\rrsub}(T) = f_{2,2}$. 

For $\Phi_1 , \Phi_2 \subseteq \Phi$, define 
\begin{equation}
    \cT_3 ( f_0, [a_0, b_0], \Phi_1, \Phi_2 ) := \bigcup_{\phi_1\in \Phi_1, \phi_2 \in \Phi_2} \cT_3 ( f_0, [a_0, b_0], \phi_1, \phi_2 ).
\end{equation}
For any $\phi = (f_0,[a,b], f_1, f_2 ) \in \Phi$, and for a given positive integer $s \le b-a$, let $(t_0, ..., t_s)$ be almost equi-spaced in $[a,b]$, and define $\phi^{s,j} := (f_0, [t_{j-1}, t_j], f_1, f_2)$ for any $j\in [s]$.

\subsection{Quantile-based pruning}
In this section, we focus on a subset of trees $\cT_3 ( f_0, [a_0, b_0], \Phi_1, \Phi_2 ) $, and discuss how to replace it with a smaller search space without missing the optimal solution. 

\begin{proposition}\label{prop: pruning-d=3}
Let $W$ and $V$ be two functions on $2^{[n]} \times \Phi$ satisfying $W(\cI, \phi) \le L_2(\cI, \phi) \le V(\cI, \phi)$ for all $\cI \subseteq [n]$ and $\phi\in \Phi $. Let $U$ be an upper bound of the optimal value of \eqref{eq:opt-d=3}. Given a subset $\cT_3 ( f_0, [a_0, b_0], \Phi_1, \Phi_2 ) $ with $f_0\in [p]$, $0\le a_0 \le b_0 \le u(f_0)$ and $\Phi_1, \Phi_2 \subseteq \Phi$; let $(t_0,...,t_s)$ be almost $s$-equi-spaced in $[a_0,b_0]$. For each $j\in [s]$, define 
\begin{equation}\label{Phi_1j}
\begin{aligned}
\Phi_{1,j} &:= \Big\{ \phi_1^{s,j'} ~\Big|~ 
\phi_1 \in \Phi_1, ~ j' \in [s], ~ 
W( [n]_{[0,t_{j-1}]}^{f_0} , \phi_1^{s,j'} ) \le 
\min_{\phi\in \Phi_1} V([n]_{[0,t_{j}]}^{f_0}, \phi), ~ \\
&
\qquad \qquad \quad
W( [n]_{[0,t_{j-1}]}^{f_0}, 
\phi_1^{s,j'} ) + \min_{\phi\in\Phi_2} W([n]_{[t_j, u(f_0)]}^{f_0}, \phi) \ \le \ U \Big\}
\end{aligned}
\end{equation}
and
\begin{equation}\label{Phi_2j}
\begin{aligned}
 \Phi_{2,j} &:= \Big\{ \phi_2^{s,j''} ~\Big|~ 
\phi_2 \in \Phi_2, ~ j'' \in [s], ~ 
W( [n]_{[t_j, u(f_0)]}^{f_0} , \phi_2^{s,j''} ) \le 
\min_{\phi\in \Phi_2} V([n]_{[t_{j-1}, u(f_0)]}^{f_0}, \phi), ~ \\
&
\qquad \qquad \quad
W( [n]_{[t_j, u(f_0)]}^{f_0}, 
\phi_2^{s,j''} ) + \min_{\phi\in\Phi_1} W([n]_{[0,t_{j-1}]}^{f_0}, \phi) \ \le \ U \Big\}.
\end{aligned}
\end{equation}
Then any optimal solution of \eqref{eq:opt-d=3} is not in 
\begin{equation}\label{opt-not-in-d=3}
\cT_3 (f_0, [a_0, b_0], \Phi_1, \Phi_2) \setminus \bigcup_{j=1}^s \cT_3 ( f_0, [t_{j-1}, t_j], \Phi_{1,j}, \Phi_{2,j} ).
\end{equation}
\end{proposition}
\begin{proof}
Let $T^*$ be any optimal solution of \eqref{opt-not-in-d=3}. 
 It suffices to prove that: if $T^* \in \cT_3 (f_0, [a_0, b_0], \Phi_1, \Phi_2)$, then there exists $j\in [s] $ such that 
\begin{equation}
    T^* \in \cT_3 (f_0, [t_{j-1}, t_j], \Phi_{1,j}, \Phi_{2,j}).
\end{equation}
Suppose $ T^* \in \cT_3 (f_0, [a_0, b_0], \Phi_1, \Phi_2)$, then $t_{\rootsub}(T^*) \in [a_0, b_0]$, and there exist integers $a_1, b_1, a_2, b_2$ such that $t_{\lsub}(T^*) \in [a_1, b_1]$, $ t_{\rsub}(T^*) \in [a_2, b_2]$, and
\begin{equation}\label{phi1-and-phi2}
    \begin{aligned}
     \phi_1 &:= ( f_{\lsub}(T^*) , [a_1, b_1], f_{\llsub}(T^*), f_{\lrsub}(T^*) ) \in \Phi_1, \\
     \phi_2 &:= ( f_{\rsub}(T^*), [a_2, b_2], f_{\rlsub}(T^*), f_{\rrsub} (T^*)) \in \Phi_2. \\
    \end{aligned}
\end{equation}

Since $[a_0,b_0] = \cup_{j=1}^s [t_{j-1}, t_j]$,  there exists $j\in [s]$ such that $t_{\rootsub}(T^*) \in [t_{j-1}, t_j] $. 
Let $(\ell_0, ...., \ell_s)$ be almost equi-spaced in $[a_1, b_1]$. $t_{\lsub}(T^*) \in [a_1,b_1] = \cup_{i=1}^s [\ell_{i-1}, \ell_{i}]$, so there exists $j'\in [s]$ such that $t_{\lsub}(T^*) \in [\ell_{j'-1}, \ell_{j'}] $. Note that (by definition)
\begin{equation}
    \phi_1^{s,j'} = ( f_{\lsub}(T^*), [\ell_{j'-1}, \ell_{j'}], f_{\llsub}(T^*), f_{\lrsub}(T^*) ),
\end{equation}
we have 
\begin{equation}\label{check-ineq1}
    \begin{aligned}
    & \qquad
     W( [n]_{[0, t_{j-1}]}^{f_0}, \phi_1^{s,j'} ) 
     ~\mathop{\le}\limits^{(i)}~ 
     L_2( [n]_{[0, t_{j-1}]}^{f_0}, \phi_1^{s,j'} ) 
     ~\mathop{\le}\limits^{(ii)}~
     L_2( [n]_{[0, t_0(T^*)]}^{f_0}, \phi_1^{s,j'} )  \\
     &
     ~\mathop{=}\limits^{(iii)}~ 
     \min_{\phi \in \Phi_1} L_2([n]_{[0, t_0(T^*)]}^{f_0},\phi) 
     ~\mathop{\le}\limits^{(iv)}~
     \min_{\phi \in \Phi_1} L_2([n]_{[0, t_j]}^{f_0},\phi) 
     ~\mathop{\le}\limits^{(v)}~ 
     \min_{\phi \in \Phi_1} V([n]_{[0, t_j]}^{f_0},\phi),
    \end{aligned}
\end{equation}
where $(i)$ is because $W(\cI , \phi) \le L_2(\cI , \phi)$ for any $\cI\subseteq [n]$ and $\phi\in \Phi$; $(ii)$ is because of $t_{\rootsub}(T^*) \in [t_{j-1}, t_j]$ and Lemma~\ref{lemma:L2-super-add}; 
$(iii) $ is because $T^*$ is the optimal solution of \eqref{eq:opt-d=3}, and the left subtree of $T^*$ (rooted at $N_L(T^*) $) is in $\cT_2(\phi_1^{s,j'})$; $(iv)$ is because of $t_{\rootsub}(T^*) \in [t_{j-1}, t_j]$ and Lemma~\ref{lemma:L2-super-add}; $(v)$ is because 
$L_2(\cI , \phi) \le V(\cI , \phi)$ for any $\cI\subseteq [n]$ and $\phi\in \Phi$. 

On the other hand, 
\begin{equation}\label{check-ineq2}
\begin{aligned}
& ~~\quad 
W([n]_{[0,t_{j-1}]}^{f_0}, \phi_1^{s,j'}) + \min_{\phi\in \Phi_2} W( [n]_{[t_j, u(f_0)]}^{f_0}, \phi ) 
\\
&\mathop{\le}\limits^{(i)} ~
W([n]_{[0,t_{j-1}]}^{f_0}, \phi_1^{s,j'}) + W([n]_{[t_j, u(f_0)]}, \phi_2) 
~\mathop{\le}\limits^{(ii)}~ 
L_2([n]_{[0,t_{j-1}]}^{f_0}, \phi_1^{s,j'}) + L_2([n]_{[t_j, u(f_0)]}^{f_0}, \phi_2) 
\\
&\mathop{\le}\limits^{(iii)}~
L_2([n]_{[0,t_0(T^*)]}^{f_0}, \phi_1^{s,j'}) + L_2([n]_{[t_0(T^*), u(f_0)]}^{f_0}, \phi_2) 
~\mathop{=}\limits^{(iv)}~ 
\min_{T\in \cT_3} \sum_{i=1}^n \ell (y_i , T(x_i)) 
~\mathop{\le}\limits^{(v)}~ U,
\end{aligned}
\end{equation}
where $(i)$ is because $\phi_2 \in \Phi_2$ (in \eqref{phi1-and-phi2}); $(ii)$ is because $W(\cI , \phi) \le L_2(\cI , \phi)$ for any $\cI\subseteq [n]$ and $\phi\in \Phi$; $(iii) $ is because of $t_0(T^*) \in [t_{j-1}, t_j]$ and Lemma~\ref{lemma:L2-super-add}; $(iv)$ is because $T^*$ is the optimal solution of \eqref{eq:opt-d=3}, the left subtree of $T^*$ (rooted at $N_L(T^*)$) is in $\cT_2(\phi_1^{s,j'})$ and the right subtree of $T^*$ (rooted at $N_R(T^*)$) is in $\cT_2(\phi_2)$; $(v)$ is because $U$ is an upper bound of the optimal value of \eqref{eq:opt-d=3}. 

 Combining \eqref{check-ineq1}, \eqref{check-ineq2} and note that $\phi_1 \in \Phi_1$ (by \eqref{phi1-and-phi2}) and $j'\in [s]$, we have 
 \begin{equation}\label{tmp-1-1}
     ( f_{\lsub}(T^*), [\ell_{j'-1}, \ell_{j'}], f_{\llsub}(T^*), f_{\lrsub}(T^*) ) = \phi_1^{s,j'} \in \Phi_{1,j}.
 \end{equation}
 Let $(r_0, ...., r_s)$ be almost equi-spaced in $[a_2, b_2]$. Since $t_{\rsub}(T^*) \in [a_2,b_2] = \cup_{i=1}^s [r_{i-1}, r_{i}]$, so there exists $j''\in [s]$ such that $t_{\rsub}(T^*) \in [r_{j''-1}, r_{j''}] $. By a similar argument as the proof of \eqref{tmp-1-1}, it can be proved that 
 \begin{equation}\label{tmp-1-2}
     ( f_{\rsub}(T^*), [r_{j''-1}, r_{j''}], f_{\rlsub}(T^*), f_{\rrsub}(T^*) )  \in \Phi_{2,j}.
 \end{equation}
 By \eqref{tmp-1-1} and \eqref{tmp-1-2}, and recall that $f_{\rootsub}(T^*) = f_0$ and $t_{\rootsub}(T^*) \in [t_{j-1}, t_j] $, we have 
 \begin{equation}
     T^* \in \cT_3 (f_0, [t_{j-1}, t_j], \Phi_{1,j}, \Phi_{2,j}).
 \end{equation}
This completes the proof of Proposition \ref{prop: pruning-d=3}. \end{proof}

Note that in the definition of $\Phi_{1,j}$, two inequalities are needed to be satisfied. 
The first inequality corresponds to a pruning procedure when only considering the left depth-2 subtree rooted at $N_L(T)$ of a tree $T$; the second inequality corresponds to a pruning procedure considering the whole depth-3 tree. A similar argument holds for the definition of $\Phi_{2,j}$. Note that this is slightly more intricate than the case for fitting depth-2 trees, where only one inequality is imposed (see Proposition \ref{prop: quantile-prune}).

\subsection{Overall strategy}
To solve \eqref{eq:opt-d=3}, \texttt{Quant-BnB} maintains and updates 
a set $\texttt{AL3}^{(k)}$ 
(short for ``alive") 
(over iterations $k=1,2,...$) that contains tuples of the form
\begin{equation*}
(f_0, [a_0,b_0], \Phi_1, \Phi_2),
\end{equation*}
where $f_0\in [p]$, $0\le a_0\le b_0\le u(f_0)$ and $\Phi_1, \Phi_2 \subseteq \Phi$. 
A tuple $(f_0, [a_0,b_0], \Phi_1, \Phi_2)$ is in $\texttt{AL3}^{(k)}$ if (based on the knowledge at iteration $k$) the optimal solution of \eqref{eq:opt-d=3} is possible to be in the set $\cT_3(f_0, [a_0,b_0], \Phi_1, \Phi_2)$. 
Initially ($k=0$), all the trees in $\cT_3$ are ``alive", so we set 
\begin{equation}\label{def:AL0-d=3}
\texttt{AL3}^{(0)} = \bigcup_{f_0=1}^p \lt\{   (f_0, [0,u(f_0)], \bar \Phi_0, \bar\Phi_0)   \rt\}, 
\end{equation}
where 
\begin{equation}
    \bar \Phi_0 = \{ (f, [0,u(f)], f_1, f_2 ) ~|~ f, f_1, f_2 \in [p] \}.
\end{equation}
\texttt{Quant-BnB} also maintains and updates the best objective value that it has found so far, denoted by $U$. Initially, we set $U $ to be the value of any feasible solution of \eqref{eq:opt-d=3}. 
At every iteration $k\ge 1$, to update $\texttt{AL3}^{(k)}$ from $\texttt{AL3}^{(k-1)}$, we first set $\texttt{AL3}^{(k)} = \emptyset $. 
The algorithm then checks all elements in $\texttt{AL3}^{(k-1)}$.
For an element $(f_0, [a_0,b_0], \Phi_1, \Phi_2)$ in $\texttt{AL3}^{(k-1)}$, if $b_0-a_0$ is less than a fixed integer $s$, then the number of trees in the space is regarded as small, and the algorithm applies an exhaustive search method to examine all candidate trees in the space $\cT_3(f_0, [a_0,b_0], \Phi_1, \Phi_2)$.
Otherwise, the algorithm conducts the following steps. 

Let $(t_0,...,t_s)$ be almost $s$-equi-spaced in $[a_0,b_0]$. Let $W$ and $V$ be two functions on $2^{[n]} \times \Phi$ satisfying $W(\cI, \phi) \le L_2(\cI, \phi) \le V(\cI, \phi)$ for all $\cI \subseteq [n]$ and $\phi\in \Phi $.

\begin{itemize}
    \item (Step 1: Update upper bound) 
    Compute
    \begin{equation}
    U'= 
    \min_{0\le j \le s, \phi_1\in \Phi_1, \phi_2 \in \Phi_2} \Big\{
        V([n]_{[0,t_j]}^{f_0}, \phi_1)
        + V( [n]_{[t_j, u(f_0)]}^{f_0}, \phi_2 )
        \Big\}. \nonumber
    \end{equation}
    If $U' < U$, update $U \leftarrow U'$, and update the corresponding best tree. 
    
    \item (Step 2: Compute lower bound and prune)  
    Compute $\Phi_{1,j}$ and $\Phi_{2,j}$ as in \eqref{Phi_1j} and \eqref{Phi_2j}, and update:
    \begin{equation}
        \texttt{AL3}^{(k)} =  \texttt{AL3}^{(k)} \bigcup \Big( \cup_{j=1}^s \{(f_0, [t_{j-1}, t_j], \Phi_{1,j}, \Phi_{2,j})\} \Big). 
    \end{equation}
\end{itemize}

Above we have discussed the overall strategy which \texttt{Quant-BnB} uses for the computation of optimal regression tree with depth 3. Additional attentions need to be paid for the 
the implementation of the algorithm and the data structure. For example, to maintain the set $\Phi_1$, it may be better to classify the elements in $\Phi_1$ into groups depending on the first two components of the elements, i.e., for $(f_1, [a_1,b_1], f_{1,1}, f_{1,2})$ and $ (f_1', [a_1',b_1'], f_{1,1}', f_{1,2}')$ in $\Phi_1$, they are in the same group if $ f_1 = f_1'$ and $[a_1, b_1] = [a_1', b_1']$. These groupings can reduce the memory usage and make the computations well-organized. A similar argument holds for $\Phi_2$. 
It is also important to make use of the tree structure and reduce the computational costs of Step 1 and Step 2.

\newpage
\section{Experiments}\label{app-sect:exp}

\subsection{Data pre-processing}\label{app-sect:preprocess}
We use a collection of 16 classification and 11 regression datasets from UCI Machine Learning Repository. Unless specified in the dataset, 80\% of data are randomly selected as the training set, and the rest are used for testing. We remove all features that do not assist prediction, i.e., unique identifiers of samples and timestamps recording when data were collected. Table \ref{tab:data} presents a summary of these datasets.

\begin{table*}[htb]
\centering
\begin{adjustbox}{width=0.98\columnwidth,center}
\begin{tabular}{l|cccc|l|cccc}  \Xhline{3\arrayrulewidth}
Name   & \textit{Task} & \begin{tabular}{@{}c@{}} number of \\ samples \end{tabular} &\begin{tabular}{@{}c@{}} number of \\ features \end{tabular} & class/dim & Name & \textit{Task}                 & \begin{tabular}{@{}c@{}} number of \\ samples \end{tabular} & \begin{tabular}{@{}c@{}} number of \\ features\end{tabular} & class/dim \\ \hline
avila&\textit{C}&10430&10&12&skin&\textit{C}&196045&3&2\\
bank&\textit{C}&1097&4&2&wilt&\textit{C}&4339&5&2\\
bean&\textit{C}&10888&16&7&carbon&\textit{R}&8576&5&3\\
bidding&\textit{C}&5056&9&2&casp&\textit{R}&36584&9&1\\
eeg&\textit{C}&11984&14&2&concrete&\textit{R}&824&8&1\\
fault&\textit{C}&1552&27&7&energy&\textit{R}&15788&28&1\\
htru&\textit{C}&14318&8&2&fish&\textit{R}&726&6&1\\
magic&\textit{C}&15216&10&2&gas&\textit{R}&29386&10&1\\
occupancy&\textit{C}&8143&5&2&grid&\textit{R}&8000&12&1\\
page&\textit{C}&4378&10&5&news&\textit{R}&31715&59&1\\
raisin&\textit{C}&720&7&2&qsar&\textit{R}&436&8&1\\
rice&\textit{C}&3048&7&2&query1&\textit{R}&8000&3&1\\
room&\textit{C}&8103&16&4&query2&\textit{R}&159874&4&3\\
segment&\textit{C}&1848&18&7&&&&\\
\Xhline{3\arrayrulewidth}
\end{tabular}
\end{adjustbox}
\caption{A summary of datasets used in experiments. \textit{C} and \textit{R} refer to classification and regression task, respectively. For classification tasks class/dim refers to the number of classes; for regression tasks, class/dim refers to the dimension of the target.}
\label{tab:data}
\end{table*}

Since most state-of-the-art algorithms can only solve datasets with binary features, we perform the same pre-processing procedure as presented in \cite{lin2020generalized}. For a feature $f\in[p]$, recall that 
$w^f_1 < w^f_2 <\dots < w^f_{u(f)}$ denotes all unique values among $\{x_{i,f}\}_{i=1}^n$. We convert feature $f$ to a set of binary features $\{f_{j}\}_{j=1}^{u(f)-1}$ defined as
\begin{equation}
    x_{i,f_j} = \left\{
    \begin{array}{ll}
        0 & \text{if } x_{i,f} < (w^f_{j} + w^f_{j+1}) / 2,\\
        1 & \text{otherwise.}
    \end{array}
    \right.
\end{equation} 
Combining $\{f_{j}\}_{j=1}^{u(f)-1}$ for each $f\in[p]$ yields a dataset with $\sum_{f\in[p]}(u(f)-1)$ binary features. The conversion is equivalent, namely, an optimal tree on the pre-processed dataset can be converted to an optimal tree on the original dataset and vice versa. In the worst case, the converted dataset has $O(np)$ binary features. It is often computationally challenging to solve the optimal decision tree on such a dataset. 

An alternative to equivalent-conversion is approximate conversion, which can greatly reduce the number of binary features. However, approximate conversion weakens the characterization capability of the tree model, which may result in non-negligible decrease in training accuracy. Hence, there is a trade-off between training accuracy and computational cost. 

We compare the equivalent-conversion conducted in our numerical experiments with an approximate binarising method adopted in \cite{JMLR:v23:20-520} (denoted by MDLP). MDLP uses a supervised discretisation algorithm based on the minimum description length principle. We take the training accuracy of CART \cite{breiman1984classification} on original datasets as a benchmark. The size of converted datasets and training accuracy are shown in Table \ref{tab:binary}. For dataset ``avila'', ``bean'',``fault'' ``room'', ``segment'' and ``skin'', the training accuracy on approximate datasets generated by MDLP is close to the optimal accuracy. However, the accuracy on original datasets outperforms that on approximate datasets by a large margin in rest cases. For several datasets, even the training accuracy of CART is comparable to the training accuracy on approximate datasets. We thus conclude that using equivalent-conversion is indispensable for obtaining high-quality trees.

\begin{table*}[htb]
\centering
\begin{adjustbox}{width=0.98\columnwidth,center}
\begin{tabular}{c|ccc|ccc|ccc} \Xhline{3\arrayrulewidth}
\multirow{2}{*}{Name} & \multirow{2}{*}{\begin{tabular}{@{}c@{}}continuous \\ feature\end{tabular} } & \multirow{2}{*}{\begin{tabular}{@{}c@{}} equivalent \\ conversion \end{tabular}} & \multirow{2}{*}{\begin{tabular}{@{}c@{}}MDLP \\ conversion \end{tabular}} & \multicolumn{3}{c|}{depth=2} & \multicolumn{3}{c}{depth=3} \\
                      &                    &                     &                     & Opt       & approx & CART         & Opt       & approx & CART       \\ \hline
avila & 10 & 22176 & 2122 & 54.2\% & 54.1\% & 50.7\% & 58.5\% & 58.5\% & 53.2\% \\ 
bank & 4 & 4078 & 26 & 92.5\% & 92.3\% & 90.9\% & 98.3\% & 97.3\% & 93.3\% \\ 
bean & 16 & 170481 & 428 & 66.3\% & 66.2\% & 65.7\% & 87.1\% & 86.9\% & 77.7\% \\ 
bidding & 9 & 10240 & 44 & 98.1\% & 98.1\% & 98.1\% & 99.3\% & 98.1\% & 98.1\% \\ 
eeg & 14 & 5239 & 118 & 66.5\% & 65.3\% & 62.5\% & 70.8\% & 68.8\% & 66.6\% \\ 
fault & 27 & 16327 & 244 & 58.3\% & 58.3\% & 54.0\% & 68.2\% & 67.3\% & 55.3\% \\ 
htru & 8 & 101412 & 92 & 97.9\% & 97.8\% & 97.7\% & 98.1\% & 97.9\% & 97.9\% \\ 
magic & 10 & 120435 & 122 & 80.5\% & 80.2\% & 79.4\% & 83.1\% & 81.1\% & 80.1\% \\ 
occupancy & 5 & 8339 & 122 & 98.9\% & 98.9\% & 98.9\% & 99.4\% & 99.1\% & 98.9\% \\ 
page & 10 & 8175 & 50 & 95.4\% & 95.1\% & 95.4\% & 97.1\% & 96.6\% & 96.4\% \\ 
raisin & 7 & 5032 & 18 & 87.4\% & 86.5\% & 86.8\% & 89.4\% & 87.5\% & 86.9\% \\ 
rice & 7 & 19982 & 28 & 93.3\% & 93.2\% & 93.0\% & 93.8\% & 93.4\% & 93.3\% \\ 
room & 16 & 2879 & 144 & 94.6\% & 94.6\% & 93.2\% & 99.2\% & 99.2\% & 96.8\% \\ 
segment & 18 & 13129 & 145 & 57.5\% & 57.4\% & 43.0\% & 88.7\% & 88.1\% & 57.4\% \\ 
skin & 3 & 765 & 108 & 92.7\% & 92.7\% & 90.7\% & 96.9\% & 96.8\% & 96.6\% \\ 
wilt & 5 & 20329 & 7 & 99.1\% & 98.7\% & 99.1\% & 99.6\% & 98.7\% & 99.3\% \\ 
\Xhline{3\arrayrulewidth}
\end{tabular}
\end{adjustbox}
\caption{Training accuracy (in percentage) of optimal classification trees with depth $2$ and $3$. The third and forth columns provide the numbers of binary features of datasets processed by equivalent-conversion and MDLP, respectively. Opt denotes the training accuracy of the optimal classification tree on original datasets, approx denotes the training accuracy of the optimal classification tree on datasets binarised using MDLP, and CART denotes the training accuracy of CART on original datasets.}
\label{tab:binary}
\end{table*}

\subsection{Optimization algorithms}

Details and experimental settings of all comparison algorithms are stated below. Unless specified, implementations of algorithms used in our experiments are obtained from their original authors.

\begin{itemize}

\item \textbf{Quant-BnB}: Our proposed algorithm is written in Julia programming language (v1.6). The parameter $s$ in Algorithm \ref{alg:bb-depth2} is set to be $3$. In Section \ref{sect:numexp2} and Section \ref{sect:numexp3}, we choose $W_{1,s'}$ defined in \eqref{def:W1} as the lower bound required in Proposition \ref{prop: quantile-prune}. The parameter $s'$ is dynamically selected as $\lfloor\frac{0.6ns}{b-a}\rfloor$ for tuple $\phi=(f_0,[a,b],f_1,f_2)$.

\item \textbf{CART} \cite{breiman1984classification}: We utilize the implementation from Python package scikit-learn.

\item \textbf{TAO} \cite{carreira2018alternating}: We implement TAO in Julia 1.6. TAO uses the solution generated by CART as a warm start.

\item \textbf{OCT} and \textbf{ORT} \cite{bertsimas2019machine}
: Since the original code is not available, we implement both methods in Python and call Gurobi9 to solve MIP models. Both methods takes the solution generated by CART as a warm start.

\item \textbf{BinOCT} \cite{verwer2019learning}: BinOCT is written in Python. We slightly modify the code so that the MIP model is solved by Gurobi9. BinOCT uses the solution generated by CART as a warm start.

\item \textbf{DL8.5} \cite{aglin2020learning}: DL8.5 is written in C++ and is run as an extension of Python.

\item \textbf{MurTree} \cite{JMLR:v23:20-520}: MurTree is written in C++ and run as an executable.

\item \textbf{FlowOCT} and \textbf{BenderOCT} \cite{aghaei2021strong}: Both methods are implemented in Python. MIP models are solved by Gurobi9.
 
\item \textbf{GOSDT} \cite{lin2020generalized}: GOSDT is written in C++ and run as an executable. GOSDT does not force hard constraints on depth, but instead applies a sparsity coefficient $\alpha$ to control the complexity. As $\alpha$ decrease, GOSDT takes longer time to solve an optimal tree. To facilitate a fair comparison with our algorithm on learning optimal depth-$2$ (or $3$) tree, we test GOSDT with
\begin{equation}
    \alpha \in \{0.1,0.05,0.02,0.01,0.005,0.002,0.001\},
\end{equation}
and select the smallest $\alpha$ with which GOSDT can learn an optimal tree with depth not greater than $2$ (or $3$). 

\end{itemize}

Other works for learning optimal trees (e.g., \cite{aghaei2019learning}) that do not show noticeable speed advantages are not mentioned above. We do not consider the comparison with these algorithms as we focus on the efficiency of solving optimal trees, 

In addition to BinOCT, MurTree and DL8.5, we also run OCT, FlowOCT, BenderOCT and GOSDT on collected classification datasets. For FlowOCT, BenderOCT and GOSDT, we convert original datasets to binary ones using equivalent-conversion described in  Section \ref{app-sect:preprocess}. None of these methods is able to learn an optimal tree on any of 16 classification datasets, so the results are not displayed.

\subsection{Results on regression tasks}
We also compare our algorithm with ORT \cite{bertsimas2019machine} on 11 regression datasets. To our best knowledge, ORT is the one of the most effective framework reported in the literature for solving optimal decision trees on regression tasks. The results are displayed in Table \ref{tab:regress}. \texttt{Quant-BnB} successfully solves trees of depth $2$ in less than 10 seconds on datasets with thousands of instances, and computes trees of depth $3$ in a reasonable time for most cases. In contrast, ORT cannot optimally solve any example in 4 hours. The experiment again confirms the advantage of \texttt{Quant-BnB} for solving shallow decision trees on relatively large-scale datasets. The scalability and versatility of our proposed method contribute to the wide applications of optimal decision trees. 
\begin{table*}[htb]
\centering
\begin{tabular}{ll|cc|cc} \Xhline{3\arrayrulewidth}
\multirow{2}{*}{Name} & \multirow{2}{*}{(n,p)} & \multicolumn{2}{c|}{depth=2}        & \multicolumn{2}{c}{depth=3}        \\
                      &                    & Quant-BnB & ORT  & Quant-BnB & ORT \\ \hline
carbon&(8576,5)&\textbf{ 0.7 }&(20\%)&\textbf{ 729 }&(423\%)\\
casp&(36584,9)&\textbf{ 4.2 }&(14\%)&\textbf{ 7609 }&(12\%)\\
concrete&(824,8)&\textbf{ \textless 0.1 }&(1.2\%)&\textbf{ 125 }&(33\%)\\
energy&(15788,28)&\textbf{ 14 }&(9.6\%)&-&(7.6\%)\\
fish&(726,6)&\textbf{ \textless 0.1 }&(4.5\%)&\textbf{ 34 }&(36\%)\\
gas&(29386,10)&\textbf{ 1.5 }&(1358\%)&\textbf{ 421 }&(510\%)\\
grid&(8000,12)&\textbf{ 1.2 }&(9.3\%)&\textbf{ 1293 }&(31\%)\\
news&(31715,59)&\textbf{ 349 }&(22\%)&-&(29\%)\\
qsar&(436,8)&\textbf{ \textless 0.1 }&(5.1\%)&\textbf{ 30 }&(32\%)\\
query1&(8000,3)&\textbf{ \textless 0.1 }&(38\%)&\textbf{ 34 }&(73\%)\\
query2&(159874,4)&\textbf{ 9.8 }&(97\%)&\textbf{ 2896 }&OoM\\
\Xhline{3\arrayrulewidth}
\end{tabular}
\caption{Comparison of \texttt{Quant-BnB} against ORT. For each dataset, the number of observations and the number of features are provided. Each entry denotes running time in seconds. - refers to time out (4h), OoM refers to out of memory (25GB).  Since ORT times out in all cases, we display the relative differences $(L_O-L_Q)/L_Q$ as a percentage instead, where $L_O$ and $L_Q$ are the training errors of ORT and \texttt{Quant-BnB}, respectively.}
\label{tab:regress}
\end{table*}

\bibliographystyle{plain}     
	\bibliography{opttree}

\begin{thebibliography}{10}

\bibitem{aghaei2019learning}
Sina Aghaei, Mohammad~Javad Azizi, and Phebe Vayanos.
\newblock Learning optimal and fair decision trees for non-discriminative
  decision-making.
\newblock In {\em Proceedings of the AAAI Conference on Artificial
  Intelligence}, volume~33, pages 1418--1426, 2019.

\bibitem{aghaei2021strong}
Sina Aghaei, Andr{\'e}s G{\'o}mez, and Phebe Vayanos.
\newblock Strong optimal classification trees.
\newblock {\em arXiv preprint arXiv:2103.15965}, 2021.

\bibitem{aglin2020learning}
Ga{\"e}l Aglin, Siegfried Nijssen, and Pierre Schaus.
\newblock Learning optimal decision trees using caching branch-and-bound
  search.
\newblock In {\em Proceedings of the AAAI Conference on Artificial
  Intelligence}, volume~34, pages 3146--3153, 2020.

\bibitem{angelino2017learning}
Elaine Angelino, Nicholas Larus-Stone, Daniel Alabi, Margo Seltzer, and Cynthia
  Rudin.
\newblock Learning certifiably optimal rule lists for categorical data.
\newblock {\em arXiv preprint arXiv:1704.01701}, 2017.

\bibitem{bennett1996optimal}
Kristin~P Bennett and Jennifer~A Blue.
\newblock Optimal decision trees.
\newblock {\em Rensselaer Polytechnic Institute Math Report}, 214:24, 1996.

\bibitem{bertsimas2017optimal}
Dimitris Bertsimas and Jack Dunn.
\newblock Optimal classification trees.
\newblock {\em Machine Learning}, 106(7):1039--1082, 2017.

\bibitem{bertsimas2019machine}
Dimitris Bertsimas and Jack Dunn.
\newblock {\em Machine learning under a modern optimization lens}.
\newblock Dynamic Ideas LLC, 2019.

\bibitem{bessiere2009minimising}
Christian Bessiere, Emmanuel Hebrard, and Barry O’Sullivan.
\newblock Minimising decision tree size as combinatorial optimisation.
\newblock In {\em International Conference on Principles and Practice of
  Constraint Programming}, pages 173--187. Springer, 2009.

\bibitem{breiman1984classification}
L~Breiman, JH~Friedman, R~Olshen, and CJ~Stone.
\newblock Classification and regression trees.
\newblock 1984.

\bibitem{carreira2018alternating}
Miguel~A Carreira-Perpin{\'a}n and Pooya~and Tavallali.
\newblock Alternating optimization of decision trees, with application to
  learning sparse oblique trees.
\newblock {\em Advances in Neural Information Processing Systems},
  31:1211--1221, 2018.

\bibitem{chen2018optimization}
Chaofan Chen and Cynthia Rudin.
\newblock An optimization approach to learning falling rule lists.
\newblock In {\em International Conference on Artificial Intelligence and
  Statistics}, pages 604--612. PMLR, 2018.

\bibitem{chen2016xgboost}
Tianqi Chen and Carlos Guestrin.
\newblock Xgboost: A scalable tree boosting system.
\newblock In {\em Proceedings of the 22nd acm sigkdd international conference
  on knowledge discovery and data mining}, pages 785--794, 2016.

\bibitem{JMLR:v23:20-520}
Emir Demirovi{\'c}, Anna Lukina, Emmanuel Hebrard, Jeffrey Chan, James Bailey,
  Christopher Leckie, Kotagiri Ramamohanarao, and Peter~J. Stuckey.
\newblock Murtree: Optimal decision trees via dynamic programming and search.
\newblock {\em Journal of Machine Learning Research}, 23(26):1--47, 2022.

\bibitem{dobkin1997induction}
David Dobkin, Truxton Fulton, Dimitrios Gunopulos, Simon Kasif, and Steven
  Salzberg.
\newblock Induction of shallow decision trees.
\newblock {\em IEEE Trans. on Pattern Analysis and Machine Intelligence}, 1997.

\bibitem{Dua:2019}
Dheeru Dua and Casey Graff.
\newblock {UCI} machine learning repository, 2017.

\bibitem{farhangfar2008fast}
Alireza Farhangfar, Russell Greiner, and Martin Zinkevich.
\newblock A fast way to produce near-optimal fixed-depth decision trees.
\newblock In {\em Proceedings of the 10th international symposium on artificial
  intelligence and mathematics (ISAIM-2008)}, 2008.

\bibitem{gunluk2021optimal}
Oktay G{\"u}nl{\"u}k, Jayant Kalagnanam, Minhan Li, Matt Menickelly, and Katya
  Scheinberg.
\newblock Optimal decision trees for categorical data via integer programming.
\newblock {\em Journal of Global Optimization}, pages 1--28, 2021.

\bibitem{hu2020learning}
Hao Hu, Mohamed Siala, Emmanuel H{\'e}brard, and Marie-Jos{\'e} Huguet.
\newblock Learning optimal decision trees with maxsat and its integration in
  adaboost.
\newblock In {\em IJCAI-PRICAI 2020, 29th International Joint Conference on
  Artificial Intelligence and the 17th Pacific Rim International Conference on
  Artificial Intelligence}, 2020.

\bibitem{hu2019optimal}
Xiyang Hu, Cynthia Rudin, and Margo Seltzer.
\newblock Optimal sparse decision trees.
\newblock {\em Advances in Neural Information Processing Systems (NeurIPS)},
  2019.

\bibitem{laurent1976constructing}
Hyafil Laurent and Ronald~L Rivest.
\newblock Constructing optimal binary decision trees is np-complete.
\newblock {\em Information processing letters}, 5(1):15--17, 1976.

\bibitem{lin2020generalized}
Jimmy Lin, Chudi Zhong, Diane Hu, Cynthia Rudin, and Margo Seltzer.
\newblock Generalized and scalable optimal sparse decision trees.
\newblock In {\em International Conference on Machine Learning}, pages
  6150--6160. PMLR, 2020.

\bibitem{mctavish2022fast}
Hayden McTavish, Chudi Zhong, Reto Achermann, Ilias Karimalis, Jacques Chen,
  Cynthia Rudin, and Margo Seltzer.
\newblock Fast sparse decision tree optimization via reference ensembles.
\newblock 2022.

\bibitem{narodytska2018learning}
Nina Narodytska, Alexey Ignatiev, Filipe Pereira, Joao Marques-Silva, and
  IS~RAS.
\newblock Learning optimal decision trees with sat.
\newblock In {\em IJCAI}, pages 1362--1368, 2018.

\bibitem{nijssen2007mining}
Siegfried Nijssen and Elisa Fromont.
\newblock Mining optimal decision trees from itemset lattices.
\newblock In {\em Proceedings of the 13th ACM SIGKDD international conference
  on Knowledge discovery and data mining}, pages 530--539, 2007.

\bibitem{nijssen2010optimal}
Siegfried Nijssen and Elisa Fromont.
\newblock Optimal constraint-based decision tree induction from itemset
  lattices.
\newblock {\em Data Mining and Knowledge Discovery}, 21(1):9--51, 2010.

\bibitem{quinlan1986induction}
J.~Ross Quinlan.
\newblock Induction of decision trees.
\newblock {\em Machine learning}, 1(1):81--106, 1986.

\bibitem{verwer2019learning}
Sicco Verwer and Yingqian Zhang.
\newblock Learning optimal classification trees using a binary linear program
  formulation.
\newblock In {\em Proceedings of the AAAI Conference on Artificial
  Intelligence}, volume~33, pages 1625--1632, 2019.

\bibitem{NEURIPS2020_1373b284}
Haoran Zhu, Pavankumar Murali, Dzung Phan, Lam Nguyen, and Jayant Kalagnanam.
\newblock A scalable mip-based method for learning optimal multivariate
  decision trees.
\newblock In H.~Larochelle, M.~Ranzato, R.~Hadsell, M.~F. Balcan, and H.~Lin,
  editors, {\em Advances in Neural Information Processing Systems}, volume~33,
  pages 1771--1781. Curran Associates, Inc., 2020.

\end{thebibliography}

\end{document}